\documentclass[10pt]{article}

\usepackage{mystyle}
\RequirePackage[top=2.25cm, bottom=2.5cm, left=2.5cm, right=2.5cm, columnsep=0.65cm]{geometry}
\usepackage{float}
\usepackage{tikz}

\title{\LARGE Toward the Optimal Regret--Instability Trade-off in Multi-Armed Bandits}
\author{
Kaifei Wang\thanks{Peking University. Email: \texttt{wkf5094@stu.pku.edu.cn}.}
\and
Yinyu Ye\thanks{Stanford University. Email: \texttt{yinyu-ye@stanford.edu}.} \thanks{Shanghai Jiao Tong University. Email: \texttt{han.zhong@sjtu.edu.cn}.}
\and
Han Zhong\footnotemark[3]
}
\date{}

\begin{document}
\maketitle

\begin{abstract}
Multi-armed bandit algorithms are evaluated by regret, yet comparable regret can coexist with different allocations across independent runs. We study the trade-off between worst-case regret $\mathcal R_{K,T}$ and instability $\mathcal S_{K,T}$, defined as the largest standard deviation of a terminal pull count, for $K$ arms and $T$ rounds. We prove the finite-time lower bound $\mathcal R_{K,T}\mathcal S_{K,T}\ge C T^{3/2}$, where $C$ is independent of $K$ and $T$, under a finite-time regret condition and without the regularity assumptions imposed in the prior asymptotic analysis. We also introduce Stabilized Lower-Envelope UCB (\textup{\textsc{SLE-UCB}}), a new tunable algorithm combining a running lower-envelope index with a decreasing pull-count stabilizer. \textup{\textsc{SLE-UCB}} satisfies $\mathcal R_{K,T}\mathcal S_{K,T}=O(T^{3/2}\log K)$, with an implicit constant independent of $K$ and $T$, matching the lower bound exactly in $T$ and within a logarithmic factor in $K$. To prove the instability bound, we develop a new offline top-prefix representation that removes path dependence from online decisions. Together with single-reward perturbations and the Efron--Stein inequality, this representation controls pull-count variance. Thus, regret and instability depend reciprocally on $K$, while their product has no polynomial dependence on $K$. These results resolve the open question raised in the literature concerning the sharp arm-dependent regret--instability frontier.
\end{abstract}

\section{Introduction}

Learning and decision-making often occur simultaneously. In clinical trials, online recommendation, and sequential resource allocation, each opportunity must be assigned while the rewards of the available alternatives remain unknown. The multi-armed bandit model captures these problems in its simplest form: every action produces an immediate reward while also generating information for future decisions \citep{li2010,gittins2011,villar2015}. The central challenge is to balance information acquisition with immediate reward. Regret has become the standard measure of this balance, comparing the reward collected by an algorithm with that attainable if the best arm were known in advance \citep{lattimore2020}.

Regret, however, is an average criterion. It reflects pull counts through their expectations but does not control how realized rewards and pull counts vary from one run to another. Two algorithms can therefore have similar regret while producing very different allocations across independent runs. This distinction matters because pull counts determine how many patients receive each treatment, how much traffic is assigned to each option, and how much data are available for subsequent inference.

A related concern arises at the level of realized performance. For example, in discussing MOSS \citep{audibert2009minimax}, \citet[Section~9.2]{lattimore2020} use a carefully tuned two-arm policy to illustrate how favorable average performance can coexist with markedly different outcomes across independent runs. Subsequent work formalizes this issue through regret tails and risk-sensitive policy design \citep{simchilevi2022simple,simchilevi2023tradeoff,fan2025fragility,zhu2026varianceinflation}. These studies focus on reward outcomes rather than arm allocations, but they reinforce the same message: regret alone does not control an algorithm's behavior across independent runs.

The allocation side of this issue is studied through sampling stability, rooted in \citet{lai1982least}. Broadly, stability requires the number of pulls of each arm to concentrate around a deterministic scale, thereby supporting conventional statistical inference after adaptive data collection. This property has been examined for UCB \citep{kalvit2021closer,khamaru2024inference,han2024ucb,chenlu2025adaptivity,fan2026precise} and Thompson sampling and its variants \citep{kalvit2021closer,fan2022typical,halder2025stable,han2026thompson,yan2026optimism}. Related experimental-design work studies the trade-off between online performance and statistical power \citep{simchilevi2023experimental}. These studies make clear that predictable allocation is not an automatic consequence of low regret; it is a separate property of the algorithm.

This longstanding literature raises a more structural question: can an algorithm be both regret-optimal and stable? \citet{praharaj2025instability} show that a broad class of minimax-optimal optimism-based algorithms fails to satisfy sampling stability. Subsequent work studies this question under complementary notions of stability. \citet{halder2026stabilizingbanditsusingregularization} study stabilization through regularization under a related stability notion. Taken together, these studies reveal a general tension between regret and allocation stability. \citet{chenlu2026} provide a quantitative formulation of this question by studying the Pareto trade-off between worst-case regret and pull-count instability. Specifically, for a \(K\)-armed bandit run for \(T\) rounds, let \(\mathcal R_{K,T}\) denote its worst-case regret and let \(\mathcal S_{K,T}\) denote its worst-case instability, measured by the largest standard deviation of an arm's terminal pull count. Formal definitions are provided in Section~\ref{sec:model_metrics}. Their lower bound applies to algorithms that treat identical arms symmetrically, pull every arm at least once in expectation, and achieve asymptotically sublinear worst-case regret. For every fixed \(K\), their result implies that there exists a constant \(C_0>0\), independent of \(T\), such that\footnote{Although the lower-bound construction of \citet{chenlu2026} uses only two arms, its asymptotic conclusion extends to every fixed \(K\) through the padding reduction in Remark~\ref{rem:fixed_k_reduction}.}
\begin{equation}
\label{eq:chen_lu_fixed_k_implication}
\liminf_{T\to\infty}
\frac{\mathcal R_{K,T}\mathcal S_{K,T}}{T^{3/2}}
> C_0.
\end{equation}
On the achievability side, they introduce \textup{\textsc{UCB-f}}, a tunable generalization of UCB1 \citep{auer2002finite} controlled by an exploration function \(f(\cdot)\), and establish
\begin{equation*}
\mathcal R_{K,T}=O_K\left(\sqrt{T}f(T)\right),
\qquad
\mathcal S_{K,T}=O_K\left(\frac{T\log T}{f(T)}\right).
\end{equation*}
Here \(O_K(\cdot)\) hides the dependence on \(K\). 
Multiplying these bounds gives a regret--instability product of \(O_K(T^{3/2}\log T)\). With the dependence on \(K\) left implicit, the upper bound matches the \(T^{3/2}\) polynomial dependence of the lower bound above but leaves a factor \(\log T\) gap. To further understand the \emph{optimal regret--instability trade-off} in \(K\)-armed bandits, three questions remain open: (i) whether the logarithmic gap in \(T\) can be closed, (ii) what the explicit and sharp dependence on \(K\) is, and (iii) whether a tunable algorithm can achieve this trade-off.

\vspace{4pt}
\noindent\textbf{Contributions.}
Our contributions are fourfold.
\begin{itemize}[leftmargin=*]
  \item \textbf{A finite-time lower bound.}
  Building on the asymptotic lower bound of \citet{chenlu2026}, we establish a finite-time counterpart under a mild condition that permits linear-in-\(T\) regret. Specifically, we prove
  \[
  \mathcal R_{K,T}\mathcal S_{K,T}\ge C T^{3/2},
  \]
  where \(C>0\) is independent of \(T\) and \(K\). Our proof adapts their argument without separately assuming symmetry, nondegenerate exploration, or asymptotic active learning. See Remark~\ref{rem:weaker_conditions} for details.

  \item \textbf{A new stabilized UCB design.} We introduce Stabilized Lower-Envelope UCB (\textup{\textsc{SLE-UCB}}), which combines two distinctive design ingredients: a running lower-envelope UCB index and a decreasing pull-count stabilizer. The lower envelope makes the reward-dependent component of the index monotone, while the stabilizer creates a deterministic gap between different pull counts. The stabilizer strength is tunable, allowing the algorithm to trade regret for lower allocation variability.

  \item \textbf{A sharp \(K\)-dependent frontier.}
  For \textup{\textsc{SLE-UCB}}, we establish regret and instability bounds with \(\sqrt K\) and \(1/\sqrt K\) dependence, respectively, up to logarithmic factors. Multiplying these bounds gives
  \[
  \mathcal R_{K,T}(\alpha)\mathcal S_{K,T}(\alpha)=O\left(T^{3/2}\log K\right).
  \]
  Together with the lower bound, this provides a near-optimal characterization of the regret--instability Pareto frontier for \(K\)-armed bandits, as illustrated in Figure~\ref{fig:pareto_frontier}. By varying \(\alpha\), \textup{\textsc{SLE-UCB}} traces this frontier, matching the lower bound exactly in \(T\) and up to a factor \(\log K\) in \(K\). This broader characterization both removes the \(\log T\) gap left open by \citet{chenlu2026} and settles their open question about the \(K\)-dependence to logarithmic accuracy.

  \item \textbf{A new instability analysis.} We develop an offline top-prefix representation of the online decisions made by \textup{\textsc{SLE-UCB}}. This representation removes the path dependence created by sequential arm selection and converts a single-reward perturbation into a controlled change in terminal pull counts. Combining this perturbation argument with the Efron--Stein inequality \citep{efron1981jackknife,boucheron2013concentration} yields the finite-time instability bound. To the best of our knowledge, this combination has not previously been used to control terminal pull-count variance in bandit analysis. The approach may also be of independent interest for studying random sample sizes generated by other adaptive allocation rules.
\end{itemize}

\begin{figure}[t]
\centering
\begingroup
\definecolor{frontierblue}{RGB}{0,92,155}
\definecolor{frontierorange}{RGB}{190,72,0}
\scalebox{0.88}{%
\begin{tikzpicture}[x=1.45cm,y=.90cm,>=latex,font=\footnotesize,line cap=round,
  orangebox/.style={draw=frontierorange,fill=white,rounded corners=1.2pt,
    line width=.8pt,inner xsep=3pt,inner ysep=2.5pt,font=\scriptsize,
    text=frontierorange,align=left},
  bluebox/.style={draw=frontierblue,fill=white,rounded corners=1.2pt,
    line width=.8pt,inner xsep=3pt,inner ysep=2pt,font=\scriptsize,
    text=frontierblue,align=left}]
  \draw[->,line width=.8pt] (.70,.55)--(8.35,.55)
    node[right=1pt] {$\mathcal R_{K,T}$};
  \draw[->,line width=.8pt] (.70,.55)--(.70,5.95)
    node[above=1pt] {$\mathcal S_{K,T}$};

  \coordinate (leftpoint) at (1.55,5.212);
  \coordinate (rightpoint) at (7.20,1.263);
  \coordinate (slepoint) at (3.90,{7.80/3.90+.18});
  \node[anchor=north] at (1.55,.47) {$\sqrt{KT}$};
  \node[anchor=north] at (7.20,.47) {$T$};
  \node[anchor=e] at (.62,5.212) {$T/\sqrt K$};
  \node[anchor=e] at (.62,1.263) {$\sqrt T$};

  \path[fill=black!5]
    (1.55,.58)--(7.00,.58)
    -- plot[domain=7.00:1.55,samples=90] (\x,{7.20/\x+.15}) -- cycle;

  \draw[frontierblue,line width=1.15pt]
    plot[domain=1.55:7.00,samples=100] (\x,{7.20/\x+.15});
  \draw[frontierorange,line width=1.15pt,dashed]
    plot[domain=1.55:7.20,samples=100] (\x,{7.80/\x+.18});

  \draw[frontierorange,densely dashed,line width=.48pt]
    (leftpoint)--(1.55,.55);
  \draw[frontierorange,densely dashed,line width=.48pt]
    (leftpoint)--(.70,5.212);
  \draw[frontierorange,densely dashed,line width=.48pt]
    (rightpoint)--(7.20,.55);
  \draw[frontierorange,densely dashed,line width=.48pt]
    (rightpoint)--(.70,1.263);
  
  \fill[frontierorange] (leftpoint) circle (1.55pt);
  \node[anchor=south west,font=\scriptsize,text=frontierorange]
    at (1.60,5.25) {$\alpha=1$};

  \node[orangebox]
    (slelabel) at (3.90,4.50)
    {\textup{\textsc{SLE-UCB}} upper bounds\\[1pt]
    $\mathcal R_{K,T}(\alpha)\mathcal S_{K,T}(\alpha)= O\Big( T^{3/2}\log K\Big)$};
  \draw[->,frontierorange,line width=.7pt]
    (slelabel.south)--([yshift=4pt]slepoint);
  \fill[white] (slepoint) circle (3.8pt);
  \draw[frontierorange,fill=white,line width=1pt]
    (slepoint) circle (3pt);
  \fill[frontierorange] (slepoint) circle (1.45pt);

  \node[bluebox,anchor=west]
    (lowerlabel) at (2.60,.90)
    {Lower bound:
     $\mathcal R_{K,T}\mathcal S_{K,T}=\Omega(T^{3/2})$};
  \draw[->,frontierblue,line width=.7pt]
    (lowerlabel.north)--++(0,0.7);
  \node[align=left,font=\scriptsize,text=black!58,fill=black!5,inner sep=1.5pt]
    at (2.40,2.00) {excluded by\\the lower bound};

  \draw[->,frontierorange,line width=.75pt]
    plot[domain=4.20:5.80,samples=36] (\x,{7.80/\x+.48});
  \node[font=\scriptsize,text=frontierorange,inner sep=1pt]
    at (5.1,2.40) {increasing $\alpha$};
\end{tikzpicture}
}%
\endgroup
\caption{Schematic \(K\)-dependent regret--instability product comparison, with logarithmic factors and constants suppressed in the axis annotations. The solid blue and dashed orange curves denote the lower bound and the \textup{\textsc{SLE-UCB}} upper bound, respectively. In the regret regime covered by Theorem~\ref{thm:lower_bound}, the gray region is excluded and the product bounds differ by at most a factor \(\log K\).}
\label{fig:pareto_frontier}
\end{figure}
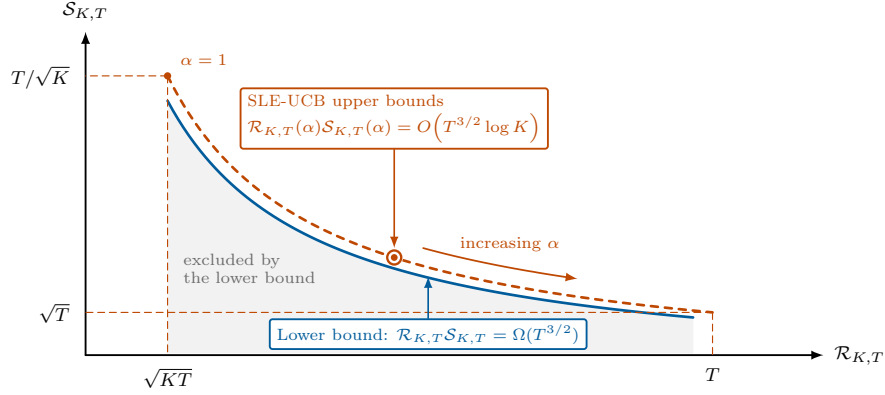

\noindent\textbf{Organization.} Section~\ref{sec:model_metrics} introduces the bandit model and performance criteria. Section~\ref{sec:fundamental_limit} presents the lower bound, the \textup{\textsc{SLE-UCB}} algorithm, and the resulting frontier comparison. Section~\ref{sec:concentration} establishes the clipping-bias and concentration results. Sections~\ref{sec:regret} and~\ref{sec:instability} provide the regret and instability analyses, respectively. Section~\ref{sec:lower_bound_comparison} proves the \(K\)-armed lower bound, and Section~\ref{sec:conclusion} concludes.

\section{Model and performance metrics}
\label{sec:model_metrics}
\noindent\textbf{Multi-armed bandits.} Consider a stochastic multi-armed bandit with $K\ge2$ arms. For each arm $i\in[K]$, let $\{X_{i,n}\}_{n\ge1}$ be an i.i.d. reward sequence with mean $\mu_i$, where $X_{i,n}$ is the reward revealed on the $n$-th pull of arm $i$. These reward sequences are mutually independent across arms. Let $\mu^\star:=\max_{i\in[K]}\mu_i$ denote the optimal mean reward, and let $\nu=(\nu_1,\ldots,\nu_K)$ denote the bandit instance, where $\nu_i$ is the distribution of $X_{i,1}$.
We make the following assumption on the reward distributions.
\begin{assumption}[{Uniformly bounded means and sub-Gaussian rewards}]
\label{assumption:bounded_mean}
There is a constant $M>0$ such that $\mu_i\in[-M,M]$ for every arm $i\in[K]$ and there is a constant $\sigma>0$ such that {the reward distribution of each arm is} $\sigma$-sub-Gaussian, in the sense that
for all $u\in\mathbb R$ and $i\in[K]$,
\begin{equation*}
\mathbb{E}\!\left[\exp\left(u(X_{i,1}-\mu_i)\right)\right]
\le
\exp\left(\frac{\sigma^2u^2}{2}\right).
\end{equation*}
\end{assumption}

\noindent\textbf{Online learning and objective.} The reward distributions are unknown, and the algorithm learns online from bandit feedback. Fix a number of rounds $T>K$. In each round $t\in[T]$, the algorithm selects an arm $A_t\in[K]$ based on the past and its internal randomization, and observes only the selected arm's reward. For every $i\in[K]$ and $t\in[T]$, define
\begin{equation*}
N_i(t):=\sum_{s=1}^t \mathbf{1}\{A_s=i\},
\qquad N_i(0):=0.
\end{equation*}
Thus, the reward observed at time $t$ is $X_{A_t,N_{A_t}(t)}$.
We evaluate an algorithm along two dimensions: regret and instability. Regret captures learning performance, whereas instability, following \citet{chenlu2026}, captures the variability of pull counts.
\begin{definition}[Regret and instability]
\label{def:regret_instability}
For a given instance $\nu$ and a given algorithm, regret and instability with $K$ arms over $T$ rounds are defined as
\begin{equation*}
\mathcal{R}_{K,T}(\nu) = \sum_{t=1}^T \mathbb E\left[\mu^\star-\mu_{A_t}\right], \quad \mathcal{S}_{K,T}(\nu) = \max_{i\in[K]}\sqrt{\operatorname{Var}[N_i(T)]}.
\end{equation*}
\end{definition}
The regret $\mathcal{R}_{K,T}(\nu)$ measures the expected cumulative reward loss compared to the optimal arm. The instability $\mathcal{S}_{K,T}(\nu)$ is the largest standard deviation of an arm's pull count across independent runs. Smaller instability therefore means that the algorithm produces more reproducible allocations across independent runs.
We evaluate a bandit algorithm by its worst-case performance over the full class \(\mathcal E\) of \(K\)-armed instances satisfying Assumption~\ref{assumption:bounded_mean}.
\begin{definition}[Worst-case regret and instability]
\label{def:worst_case_regret_instability}
For a given algorithm, the worst-case regret and instability with $K$ arms over $T$ rounds are defined as
\begin{equation*}
\mathcal{R}_{K,T}(\mathcal{E})=\sup_{\nu\in\mathcal{E}} \mathcal{R}_{K,T}(\nu),
\qquad
\mathcal{S}_{K,T}(\mathcal{E})=\sup_{\nu\in\mathcal{E}} \mathcal{S}_{K,T}(\nu).
\end{equation*}
\end{definition}
We henceforth omit the instance-class $\mathcal{E}$ and simply write $\mathcal{R}_{K,T}$ and $\mathcal{S}_{K,T}$ for these worst-case quantities.

\section{Fundamental limit and Stabilized Lower-Envelope UCB}
\label{sec:fundamental_limit}

In this section, we characterize the regret--instability trade-off for \(K\)-armed bandits, including its dependence on both the number of rounds \(T\) and the number of arms \(K\). We first establish a lower bound and then present an achievability result that matches it exactly in \(T\) and up to a logarithmic factor in \(K\).

\subsection{\texorpdfstring{\(K\)-armed lower bound}{K-armed lower bound}}
While \citet{chenlu2026} establish the \(T^{3/2}\) regret--instability lower bound only asymptotically, under symmetry, nondegenerate exploration, and an asymptotic active-learning condition, the following theorem holds for every \(K\ge2\) and \(T>K\). It requires the regret condition \(\mathcal R_{K,T}\le MT/8\), without separately imposing any of those three conditions.

\begin{theorem}[\(K\)-armed regret--instability lower bound]
\label{thm:lower_bound}
Fix \(K\ge2\) and \(T>K\). Let \(\mathcal R_{K,T}\) and \(\mathcal S_{K,T}\) denote the worst-case regret and instability as in Definition~\ref{def:worst_case_regret_instability}. For the class of $K$-armed instances satisfying Assumption~\ref{assumption:bounded_mean}, every algorithm whose worst-case regret satisfies \(\mathcal R_{K,T}\le MT/8\) obeys
\begin{equation}
\begin{aligned}
\mathcal R_{K,T}\mathcal S_{K,T}
&\geq C T^{3/2}.
\end{aligned}
\label{eq:lower_bound}
\end{equation}
Here \(C=C(M,\sigma)>0\) depends only on the instance-class parameters $(M, \sigma)$ and is independent of $(T, K)$.
\end{theorem}

\begin{proof}
See Section~\ref{sec:lower_bound_comparison} for a detailed proof.
\end{proof}

Our proof largely follows the two-armed proof of \citet{chenlu2026} but uses a refined analysis to establish the result under weaker conditions; see Section~\ref{sec:lower_bound_comparison} for details.
Theorem~\ref{thm:lower_bound} makes the trade-off explicit. At regret scale \(O(\alpha\sqrt{KT})\), instability must be \(\Omega(T/(\alpha\sqrt K))\). In particular, at the minimax regret rate \(O(\sqrt{KT})\), attained for example by \textup{\textsc{MOSS}} \citep{audibert2009minimax}, instability must be \(\Omega(T/\sqrt K)\). Thus, a minimax-regret policy cannot keep terminal pull counts tightly concentrated across all instances, consistent with the instability established for a broad class of minimax-optimal optimism-based algorithms by \citet{praharaj2025instability}.
We next turn to the achievability side and propose an algorithm matching this scaling up to logarithmic factors.

\begin{algorithm}[t]
  \caption{Stabilized Lower-Envelope UCB (\textup{\textsc{SLE-UCB}})}
  \label{alg:lower_envelope_ucb}
  \begin{algorithmic}[1]
  \STATE \textbf{Input:} number of rounds $T$, number of arms $K$, positive constants $M,\sigma$, and tuning parameter $\alpha\ge1$.
  \STATE Compute $\iota=\sqrt{\log(KT)}$, $B=M+4\sigma\iota$, $\beta=8B\iota$, and $\lambda=\alpha\beta\sqrt{K/T}$.
  \STATE Sample a uniform random priority permutation $\omega\in\mathfrak S_K$.
  \STATE \textbf{Initialization:} Pull each arm once, observe the rewards and clip them.\\
   After initialization, each arm has one pull, so $N_i(K)=1$ for all $i$.
  \FOR{$t=K+1,\ldots,T$}
  \STATE Compute $U_i(N_i(t-1))$ from \eqref{eq:lower_envelope_ucb} and $I_i(N_i(t-1))$ from \eqref{eq:stabilized_index} for each arm $i$.
  \STATE Choose $A_t\in\arg\max_{i\in[K]} I_i(N_i(t-1))$, breaking ties by the smallest priority rank $\omega(i)$.
  \STATE Pull $A_t$, observe the raw reward from that arm, clip it, and update the corresponding count.
  \ENDFOR
  \end{algorithmic}
  \end{algorithm}
\subsection{Stabilized Lower-Envelope UCB}
\label{sec:sle_ucb}
We propose Stabilized Lower-Envelope UCB (\textup{\textsc{SLE-UCB}}), a tunable algorithm for balancing regret and instability.
It is a UCB-type algorithm with two design components: (i) replacing the standard empirical-mean-plus-bonus statistic with its running lower envelope; and (ii) adding a logarithmic stabilizer that decreases with the pull count.
These two components together create a stabilized index for each arm. At each time, the algorithm selects the arm with the largest stabilized index.

Raw rewards are clipped before computing the index. For each arm \(i\in[K]\) and \(s\in[T]\), define
\begin{equation}
\label{eq:iota_def}
Y_{i,s}:=\min\{B,\max\{-B,X_{i,s}\}\},\quad
B=M+4\sigma\iota,\quad \iota=\sqrt{\log(KT)}.
\end{equation}
When arm \(i\) is pulled for the \(s\)-th time, the algorithm uses \(Y_{i,s}\) in place of the raw reward \(X_{i,s}\). The threshold \(B\) depends only on $M$, $\sigma$, \(K\), and \(T\), and Lemma~\ref{lem:clipping_bias} shows that the resulting bias is negligible. Let \(Y=(Y_{i,s})_{i\in[K],\,s\in[T]}\) denote the full array of independent clipped rewards, revealed sequentially as the arms are pulled. For \(n\ge1\), define the average clipped reward by \(\overline Y_i(n)=(\sum_{s=1}^nY_{i,s})/n\).
The lower-envelope UCB statistic is defined as
\begin{equation}
U_i(n)=\min_{1\le s\le n}\left\{\overline Y_i(s)+\frac{\beta}{\sqrt s}\right\},
\quad \text{where } \beta=8B\iota.
\label{eq:lower_envelope_ucb}
\end{equation}
The lower envelope is the minimum of the standard UCB statistic over all previous prefixes. Thus \(U_i(n)\) is nonincreasing in \(n\).
For a tuning parameter \(\alpha\ge1\), define the stabilized index by
\begin{equation}
I_i(n)=U_i(n)+\lambda\log\bigg(\frac{T}{n}\bigg),
\quad \text{where } \lambda=\alpha\beta\sqrt{\frac{K}{T}}.
\label{eq:stabilized_index}
\end{equation}
Since $U_i(n)$ is nonincreasing in $n$ and $\lambda\log(T/n)$ is strictly decreasing in $n$, the stabilized index $I_i(n)$ is a strictly decreasing sequence in $n$ for each arm $i$.
  During the initial period \(t\le K\), \textup{\textsc{SLE-UCB}} pulls each arm once. For \(t>K\), it chooses an arm with the largest stabilized index and resolves ties using the sampled priority order:
  \begin{equation}
  A_t\in\operatorname*{arg\,max}_{i\in[K]} I_i(N_i(t-1)).
  \label{eq:arm_selection}
  \end{equation}
We impose a priority order on the arms to handle numerical ties in the index values. To avoid arm-label bias, this order is chosen at random: before starting, \textup{\textsc{SLE-UCB}} samples an independent uniform permutation $\omega\in\mathfrak S_K$ and breaks ties in favor of the arm with the smallest rank $\omega(i)$.

In Section~\ref{sec:online_offline_equivalence}, we show that the monotonicity of $I_i(n)$ permits an offline view of the online selection rule, which is crucial for the analysis of instability. This equivalence is formalized in Lemma~\ref{lem:online_offline_equivalence}.
Moreover, the logarithmic stabilizer $\lambda\log(T/n)$ gives a deterministic lower bound on the index gap between two pull counts of the same arm. Specifically, for all $1\le n_1 \le n_2\le T$,
\begin{equation}
I_i(n_1)-I_i(n_2)
=
U_i(n_1)-U_i(n_2)+\lambda\log\frac{n_2}{n_1}
\ge
\lambda\log\frac{n_2}{n_1} \ge 0.
\label{eq:index_gap}
\end{equation}
In the instability proof, this gap converts a bound on a single-reward index perturbation into a bound on the resulting change in the terminal pull count; see Section~\ref{sec:term_starstar_bound}.
Although the logarithmic stabilizer may influence arm selection, Section~\ref{sec:regret} shows that it only contributes a controlled additional term to regret.

\subsection{Achievability and the Pareto frontier}
\label{sec:main_theorem}
The next theorem quantifies this trade-off through finite-time regret and instability bounds.
\begin{theorem}[Regret--instability guarantees for \textup{\textsc{SLE-UCB}}]
\label{thm:tunable_upper} 
Fix \(2\le K<T\) and \(\alpha \ge 1\). Recall from \eqref{eq:iota_def} and \eqref{eq:lower_envelope_ucb} that \(\iota=\sqrt{\log(KT)}\), \(B=M+4\sigma\iota\), and \(\beta=8B\iota\). For \textup{\textsc{SLE-UCB}} tuned with parameter \(\alpha\), let \(\mathcal R_{K,T}(\alpha)\) and \(\mathcal S_{K,T}(\alpha)\) denote its worst-case regret and instability as in Definition~\ref{def:worst_case_regret_instability}. These quantities satisfy
{
\begin{equation}
\mathcal{R}_{K,T}(\alpha)
\le
\min\left\{2MT,\,
\bigl(5+\alpha\log(2K)\bigr)\beta\sqrt{KT}\right\},
\qquad
\mathcal{S}_{K,T}(\alpha)
\le
2\sqrt T+\frac{18\sqrt 2\sigma T}{\alpha\beta\sqrt K}.
\label{eq:main_performance_bounds}
\end{equation}}
{Consequently, for fixed \(M\) and \(\sigma\), uniformly over \(\alpha\ge1\),
\begin{equation}
\mathcal{R}_{K,T}(\alpha)\mathcal{S}_{K,T}(\alpha)
=O\left(T^{3/2}\log K\right),
\label{eq:tunable_tilde_summary}
\end{equation}
where the implicit constant is independent of \(K\), \(T\), and \(\alpha\).}
\end{theorem}

\begin{proof}\mbox{}\par
{See Sections~\ref{sec:regret} and~\ref{sec:instability} for detailed proofs of the regret and instability bounds, respectively. For the product bound, applying the two branches of the regret estimate in~\eqref{eq:main_performance_bounds} to the corresponding terms of the instability estimate gives
\begin{align*}
\mathcal R_{K,T}(\alpha)\mathcal S_{K,T}(\alpha)
&\le 2\sqrt T\,\mathcal R_{K,T}(\alpha)
+\frac{18\sqrt2\sigma T}{\alpha\beta\sqrt K}\,
\mathcal R_{K,T}(\alpha)\\
&\le 4MT^{3/2}
+18\sqrt2\sigma\left(\frac{5}{\alpha}+\log(2K)\right)T^{3/2}\\
&=O\left(T^{3/2}\log K\right),
\end{align*}
where the last equality uses \(\alpha\ge1\) and \(K\ge2\).}
\end{proof}

Theorem~\ref{thm:tunable_upper} provides finite-time guarantees for the tunable \textup{\textsc{SLE-UCB}} algorithm over the full range \(\alpha\ge1\), and its product bound is uniform in \(\alpha\). Whenever \(\mathcal R_{K,T}(\alpha)\le MT/8\), this product matches the lower bound exactly in its \(T^{3/2}\) dependence and within a factor \(\log K\) in its dependence on the number of arms.
Since the regret admits the universal bound
\(\mathcal R_{K,T}(\alpha)\le 2MT\), we focus on the tuning range
\[
1\le \alpha
=O\left(\frac{\sqrt{T/K}}{\log K\log(KT)}\right),
\]
over which the tunable regret bound is at most \(O(T)\).
Because \(\beta=\Theta(\log(KT))\) for fixed \(M\) and \(\sigma\), the two guarantees take the simpler form
\begin{equation}
\mathcal R_{K,T}(\alpha)
=O\left(\alpha\log K\log(KT)\sqrt{KT}\right),
\qquad
\mathcal S_{K,T}(\alpha)
=O\left(\frac{T}{\alpha\sqrt K\log(KT)}\right).
\label{eq:simplified_regret_instability_bounds}
\end{equation}
As \(\alpha\) increases within this range, the regret bound rises while the instability bound falls, tracing the regret--instability frontier. Compared with the \(O_K(T^{3/2}\log T)\) product obtained from the reported \textup{\textsc{UCB-f}} bounds of \citet{chenlu2026}, our \(O(T^{3/2}\log K)\) guarantee replaces the explicit \(\log T\) factor with a \(\log K\) factor, thereby answering their open \(K\)-dependence question up to a logarithmic factor while retaining the optimal dependence on \(T\).

For comparison, the \textup{\textsc{UCB-f}} analysis \citep{chenlu2026} controls random pull counts around a deterministic fluid approximation through armwise UCB-index comparisons and concentration. The sharper \(K\)-dependence in Theorem~\ref{thm:tunable_upper} instead relies on the monotone lower-envelope indices and deterministic stabilizer gap of \textup{\textsc{SLE-UCB}}, together with the offline top-prefix representation and single-reward perturbation analysis developed in Sections~\ref{sec:online_offline_equivalence} and~\ref{sec:term_starstar_bound}. The Efron--Stein inequality (see \citet{efron1981jackknife,boucheron2013concentration} or Lemma~\ref{lem:efron_stein_new}) then converts these perturbation bounds into variance control. See Sections~\ref{sec:regret} and~\ref{sec:instability} for details.

\section{Clipping bias and concentration}
\label{sec:concentration}
In this section, we characterize the properties of the clipped reward.
\begin{lemma}[Clipping bias]
\label{lem:clipping_bias}
Let $\theta_i=\mathbb E[Y_{i,1}]$ be the expectation of the clipped reward for arm $i$. For every arm $i$, the clipping bias of mean estimation is upper bounded by
\begin{equation}
|\theta_i-\mu_i|\le \sigma(KT)^{-6}.
\label{eq:clipping_bias_bound}
\end{equation}
\end{lemma}

\begin{proof}
Let $X_{i,1}$ be a reward from arm $i$ and $Y_{i,1}$ be the corresponding clipped reward. By Jensen's inequality, we can bound the clipping bias as follows:
\begin{equation*}
|\theta_i-\mu_i|
=
\left|\mathbb E[Y_{i,1}]-\mathbb E[X_{i,1}]\right|
\le
\mathbb E[|Y_{i,1}-X_{i,1}|].
\end{equation*}
As $Y_{i,1}$ and $X_{i,1}$ can differ only when $|X_{i,1}|>B$, we have $|Y_{i,1}-X_{i,1}|=|Y_{i,1}-X_{i,1}|\cdot \mathbf 1\{|X_{i,1}|>B\}$.
By Assumption~\ref{assumption:bounded_mean}, \(|\mu_i|\le M<B\). Thus, when the reward exceeds $B$ in absolute value,
 the difference between the clipped reward and the raw reward can be bounded by
\begin{align*}
    |Y_{i,1}-X_{i,1}|\cdot \mathbf 1\{|X_{i,1}|>B\}&\leq |X_{i,1}-\mu_i|\cdot\mathbf 1\{|X_{i,1}|>B\}\\
&\leq |X_{i,1}-\mu_i|\cdot \mathbf 1\{|X_{i,1}-\mu_i|> 4\sigma\iota\}.
\end{align*}
The second inequality holds because \(|X_{i,1}|>B=M+4\sigma\iota\) and \(|\mu_i|\le M\) imply \(|X_{i,1}-\mu_i|>4\sigma\iota\).
We rewrite the truncated tail expectation using the tail-integral representation:
\begin{align}
|\theta_i-\mu_i|
&\le
\mathbb E\!\left[|X_{i,1}-\mu_i|\mathbf 1\{|X_{i,1}-\mu_i|>4\sigma\iota\}\right] \notag\\
&=
4\sigma\iota\,\mathbb P\!\left(|X_{i,1}-\mu_i|>4\sigma\iota\right)
+
\int_{4\sigma\iota}^{\infty}\mathbb P\!\left(|X_{i,1}-\mu_i|>t\right)\,\mathrm dt .
\label{eq:clipping_bias_bound_intermediate}
\end{align}
Substituting the sub-Gaussian tail bound into~\eqref{eq:clipping_bias_bound_intermediate} and evaluating the Gaussian tail integral, we have
\begin{align*}
|\theta_i-\mu_i|
&\le
8\sigma\iota e^{-8\iota^2}
+
2\int_{4\sigma\iota}^{\infty}\exp\!\left(-\frac{t^2}{2\sigma^2}\right)\mathrm dt\\
&\le
8\sigma\iota e^{-8\iota^2}
+
\frac{1}{2\sigma\iota}
\int_{4\sigma\iota}^{\infty}t\exp\!\left(-\frac{t^2}{2\sigma^2}\right)\mathrm dt \notag\\
&=
\left(8\iota+\frac{1}{2\iota}\right)\sigma e^{-8\iota^2}
\le
\sigma e^{-6\iota^2}
=
\sigma(KT)^{-6}.
\end{align*}
The last inequality holds since $8\iota+1/(2\iota)\le e^{2\iota^2}$ when $\iota=\sqrt{\log(KT)}\ge\sqrt{\log 4}$.
\end{proof}
According to Lemma~\ref{lem:clipping_bias}, the gap between the expectation of the clipped reward and the true mean is negligible, so that the UCB statistic computed from the clipped rewards still reflects the performance of the arms. In the following lemma, we show the concentration of the clipped reward.
\begin{lemma}[Clipped reward concentration]
\label{lem:clipped_concentration}
Define the concentration event $\mathcal G$ by
\begin{equation}
\mathcal G=
\left\{
\left|\overline Y_i(n)-\theta_i\right|
\le \frac{\beta}{2\sqrt n}
\text{ for all } i\in[K],\ 1\le n\le T
\right\}.
\label{eq:uniform_concentration}
\end{equation}
The event $\mathcal G$ satisfies $\mathbb P(\mathcal G^c)\le 2(KT)^{-7}$.
\end{lemma}
\begin{proof}
    Since the clipped rewards lie in $[-B,B]$, Hoeffding's inequality gives, for each fixed $i$ and $n$,
\begin{equation*}
\mathbb P\left(\left|\overline Y_i(n)-\theta_i\right|>\frac{\beta}{2\sqrt n}\right)
\le
2\exp\left(-\frac{\beta^2}{8B^2}\right).
\end{equation*}
Applying a union bound over at most $KT$ items $(i,n)$ and using $\beta=8B\iota=8B\sqrt{\log(KT)}$ gives
\begin{equation*}
\mathbb P(\mathcal G^c)\le KT\cdot \mathbb P\left(\left|\overline Y_i(n)-\theta_i\right|>\frac{\beta}{2\sqrt n}\right)\leq 2KT\cdot(KT)^{-8}=2(KT)^{-7}.
\end{equation*}
This finishes the proof of Lemma~\ref{lem:clipped_concentration}.
\end{proof}
\section{Regret Analysis}
\label{sec:regret}
\begin{proof}
Fix an admissible instance $\nu$. We prove the regret bound in~\eqref{eq:main_performance_bounds} and the worst-case statement then follows by taking the supremum over $\nu$. Decompose the regret according to the concentration event $\mathcal G$ from Lemma~\ref{lem:clipped_concentration}:
\begin{equation}\mathcal R_{K,T}(\nu)
=
\underbrace{\mathbb E\bigg[\mathbf 1_{\mathcal G}\sum_{i=1}^K(\mu^\star-\mu_i)N_i(T)\bigg]
}_{\text{Term (I)}}+ \underbrace{\mathbb E\bigg[\mathbf 1_{\mathcal G^c}\sum_{i=1}^K(\mu^\star-\mu_i)N_i(T)\bigg]}_{\text{Term (II)}}.
\label{eq:regret_3}
\end{equation}
\par\smallskip\noindent\textbf{Term (I) in~\eqref{eq:regret_3}.}
Let $\theta^\star=\max_{i\in[K]}\theta_i$ be the largest clipped mean. 
By Lemma~\ref{lem:clipping_bias}, each true mean gap is bounded above by the corresponding clipped mean gap plus twice the clipping-bias bound. Since $\sum_{i=1}^K N_i(T)=T$, we have
\begin{align}
\mathbb E\bigg[\mathbf 1_{\mathcal G}\sum_{i=1}^K(\mu^\star-\mu_i)N_i(T)\bigg]
&\le
\mathbb E\bigg[\mathbf 1_{\mathcal G}\sum_{i=1}^K(\theta^\star-\theta_i)N_i(T)\bigg]
+2\sigma T(KT)^{-6}.
\label{eq:regret_on_good_event_1}
\end{align}
It remains to bound the first term in \eqref{eq:regret_on_good_event_1}. We divide it into the contribution from the initial pulls and the contribution from the noninitial pulls.
For the initialization phase ($t\leq K$), $\theta_i\in[-B,B]$ for every arm $i$ since the rewards are clipped, so the total is at most $2BK$. Since $K\le T$ and $\beta=8B\iota>2B$, we can bound the contribution of the initialization phase as follows:
\begin{equation}
    \sum_{i=1}^K(\theta^\star-\theta_i)\le 2BK\le \beta\sqrt{KT}.
\label{eq:initial_pull_regret_bound}
\end{equation}
Now consider a noninitial time $t$, at which \textup{\textsc{SLE-UCB}} pulls arm $i$. Let $j$ be the arm satisfying $\theta_j=\theta^\star$. According to the selection rule of \textup{\textsc{SLE-UCB}}~\eqref{eq:arm_selection}, we have
\begin{equation}
I_i(N_i(t-1))\ge I_j(N_j(t-1)).
\label{eq:choice_comparison_new}
\end{equation}
On $\mathcal G$, we first upper bound the index of the selected arm $i$. Since $U_i(n)$ is the minimum over all prefixes up to $n$, evaluating it at the full prefix gives
\begin{align}
I_i(N_i(t-1))
&=
U_i(N_i(t-1))+\lambda\log\bigg(\frac{T}{N_i(t-1)}\bigg) \notag\\
&\leq
\overline Y_i(N_i(t-1))
+\frac{\beta}{\sqrt{N_i(t-1)}}
+\lambda\log\bigg(\frac{T}{N_i(t-1)}\bigg) \notag\\
&\leq
\theta_i+\frac{3\beta}{2\sqrt{N_i(t-1)}}
+\lambda\log\bigg(\frac{T}{N_i(t-1)}\bigg).
\label{eq:chosen_index_upper}
\end{align}
The first inequality follows from the definition of the lower-envelope UCB statistic~\eqref{eq:lower_envelope_ucb}, and the second inequality follows from the definition of event $\mathcal G$.
Next, for the arm $j$ satisfying $\theta_j=\theta^\star$, let
\[
s_j\in\operatorname*{arg\,min}_{1\le s\le N_j(t-1)}
\left\{\overline Y_j(s)+\frac{\beta}{\sqrt s}\right\}.
\]
Then the index of arm $j$ satisfies
\begin{align}
I_j(N_j(t-1))
&=
\overline Y_j(s_j)+\frac{\beta}{\sqrt{s_j}}
+\lambda\log\bigg(\frac{T}{N_j(t-1)}\bigg) \notag\\
&\ge
\theta_j+\frac{\beta}{2\sqrt{s_j}}
+\lambda\log\bigg(\frac{T}{N_j(t-1)}\bigg)
\ge
\theta^\star.
\label{eq:best_index_lower}
\end{align}
The first inequality follows from the event $\mathcal G$, and the last inequality uses $\theta_j=\theta^\star$ and $N_j(t-1)\le T$.

{Combining \eqref{eq:choice_comparison_new}, \eqref{eq:chosen_index_upper} and \eqref{eq:best_index_lower}, for every noninitial round $t>K$ with $A_t=i$,}
\begin{equation}
{
\theta^\star-\theta_i
\le
\frac{3\beta}{2\sqrt{N_i(t-1)}}
+
\lambda\log\bigg(\frac{T}{N_i(t-1)}\bigg).
}
\label{eq:gap_per_pull}
\end{equation}
{Summing~\eqref{eq:gap_per_pull} over all noninitial pulls, we upper bound the noninitial contribution to Term (I) by}
\begin{align}
    \sum_{i=1}^K(\theta^\star-\theta_i)(N_i(T)-1)&\leq \frac{3\beta}{2}\sum_{i=1}^{K}\sqrt{N_i(T)-1}+\lambda\!\sum_{i:N_i(T)>1} (N_i(T)-1)\log\bigg(\frac{T}{N_i(T)-1}\bigg)\label{eq:gap_non_initial_1}\\
    &\leq \frac{3\beta}{2}\sqrt{KT}+\lambda (T-K)\log \bigg(\frac{KT}{T-K}\bigg).\label{eq:gap_non_initial_2}
\end{align}
In \eqref{eq:gap_non_initial_1}, we take $t$ in~\eqref{eq:gap_per_pull} to be the last time arm $i$ is pulled, so that the previous pull count equals $N_i(T)-1$. The second inequality~\eqref{eq:gap_non_initial_2} uses $\sum_{i=1}^{K}N_i(T)=T$. The first term follows from the Cauchy--Schwarz inequality and the second term follows from the concavity of $x\mapsto x\log(T/x)$.

Combining the initial phase contribution~\eqref{eq:initial_pull_regret_bound} and the noninitial phase contribution~\eqref{eq:gap_non_initial_2}, substituting $\lambda=\alpha\beta\sqrt{K/T}$, we have
\begin{align} 
    \sum_{t=1}^T (\theta^\star-\theta_{A_t})&\leq \beta\sqrt{KT}+\frac{3\beta}{2}\sqrt{KT}+\alpha\beta\sqrt{KT}\cdot \frac{T-K}{T}\log \bigg(\frac{KT}{T-K}\bigg)\\
    &\leq \bigg(3+\alpha\log K +\frac{\alpha(T-K)}{T}\log \Big(\frac{T}{T-K}\Big)\bigg)\beta\sqrt{KT}\\
    &\leq \big(3+\alpha\log(2K)\big)\beta\sqrt{KT}.\label{eq:clipped_regret_term_bound}
\end{align}
The last line holds because $(T-K)/T <1$ and $x\log(1/x)<1/e<\log 2$ when $x<1$.
Substituting \eqref{eq:clipped_regret_term_bound} into \eqref{eq:regret_on_good_event_1}, Term (I) is bounded by
\begin{equation}
\mathbb E\bigg[\mathbf 1_{\mathcal G}\sum_{i=1}^K(\mu^\star-\mu_i)N_i(T)\bigg]
\le
\big(3+\alpha\log(2K)\big)\beta\sqrt{KT}
+2\sigma T(KT)^{-6}.
\label{eq:term_one_regret_bound}
\end{equation}

\par\smallskip\noindent{\textbf{Term (II) in~\eqref{eq:regret_3}.}}
On $\mathcal G^c$, each mean gap is at most $2M$, since the true means are bounded in $[-M,M]$. Thus, the regret is at most $2MT$. According to Lemma~\ref{lem:clipped_concentration}, the probability of $\mathcal G^c$ is at most $2(KT)^{-7}$, so Term (II) can be bounded by
\begin{equation}
    \mathbb E\bigg[\mathbf 1_{\mathcal G^c}\sum_{i=1}^K(\mu^\star-\mu_i)N_i(T)\bigg]\le 2MT\cdot 2(KT)^{-7}=4MT(KT)^{-7}.
\label{eq:term_two_regret_bound}
\end{equation}
\par\smallskip\noindent{\textbf{Combining the bounds.}}
Combining~\eqref{eq:term_one_regret_bound} and~\eqref{eq:term_two_regret_bound}, we obtain
\begin{align}
\mathcal R_{K,T}(\nu)
&\le
\big(3+\alpha\log(2K)\big)\beta\sqrt{KT}
+2\sigma T(KT)^{-6}
+4MT(KT)^{-7} \notag\\
&\le
\bigl(5+\alpha\log(2K)\bigr)\beta\sqrt{KT}.
\label{eq:regret_final}
\end{align}
The last inequality uses basic algebraic manipulations.
By the definitions $\beta=8B\iota$ and $B=M+4\sigma\iota$, together with $\iota\ge1$, we have $\sigma\le \beta$ and $M\le \beta$.
Thus, for $K,T\ge2$, we have $2\sigma T(KT)^{-6}\leq \beta\sqrt{KT}$ and $4MT(KT)^{-7}\leq \beta\sqrt{KT}$. 
Taking the supremum over all admissible instances proves \eqref{eq:main_performance_bounds}.
\end{proof}

\section{Instability Analysis}
\label{sec:instability}
This section proves the instability bound in~\eqref{eq:main_performance_bounds}. Fix an admissible instance \(\nu\) and an arm \(i\in[K]\).
Recall from Section~\ref{sec:sle_ucb} that \(Y=(Y_{a,m})_{a\in[K],\,m\in[T]}\) denotes the full array of independent clipped potential rewards. 
We use \(\mathcal G(Y)\) to denote the event that \(Y\) satisfies
the concentration condition in~\eqref{eq:uniform_concentration}.
Since \(N_i(T)\) is determined by the clipped reward array \(Y\) and the random tie-breaking permutation \(\omega\), we write the final count as \(N_i(Y,\omega,T)\).
Applying the law of total variance gives
\begin{equation} 
\operatorname{Var}(N_i(Y,\omega,T))
=
\underbrace{\operatorname{Var}_{\omega}\!\left(\mathbb E_Y[N_i(Y,\omega,T)\mid \omega]\right)}_{\text{Term ($\star$)}}
+
\underbrace{\mathbb E_{\omega}\!\left[\operatorname{Var}_Y(N_i(Y,\omega,T)\mid \omega)\right]}_{\text{Term ($\star\star$)}}.
\label{eq:var:decomp}
\end{equation}
Here and throughout this section, all randomness is induced only by the clipped reward array $Y$ and the random tie-breaking permutation $\omega$. 
The subscripts on expectations and variances specify which of these two sources of randomness is being averaged over: $\mathbb E_\omega$ and $\operatorname{Var}_\omega$ are taken only with respect to $\omega$, while $\mathbb E_Y$ and $\operatorname{Var}_Y$ are taken only with respect to $Y$; the other source of randomness is held fixed or conditioned on as displayed.
We will bound Term ($\star$) in Section~\ref{sec:term_star_bound} and Term ($\star\star$) in Section~\ref{sec:term_starstar_bound}, respectively.

The key device is an offline top-prefix representation of the \textup{\textsc{SLE-UCB}} selection rule. Conditional on the permutation $\omega$, the final count $N_i(Y,\omega,T)$ is a deterministic function of the reward array $Y$, so the Efron--Stein inequality reduces its conditional variance to {single-coordinate perturbations} of $Y$. A direct online analysis is difficult because a single perturbation may alter all subsequent decisions. The offline representation removes this path dependence, thereby enabling the required perturbation analysis. We begin by establishing the online--offline equivalence.
\subsection{Online--offline equivalence}
\label{sec:online_offline_equivalence}
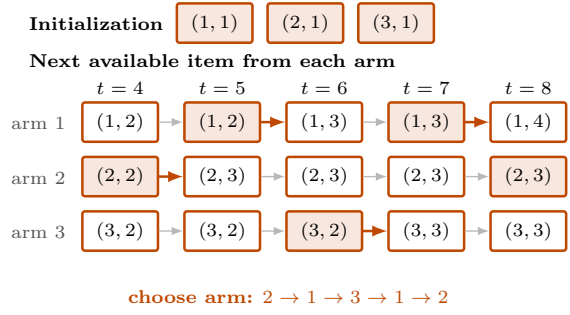
\begin{figure}[t!]
\centering
\begingroup
\definecolor{offlinecolor}{RGB}{0,92,155}
\definecolor{onlinecolor}{RGB}{190,72,0}
\begin{tikzpicture}[x=1cm,y=1cm,font=\footnotesize,>=latex,line cap=round,
  offblock/.style={draw=offlinecolor,fill=white,rounded corners=1pt,
    minimum width=1.00cm,minimum height=.52cm,inner sep=1pt,
    line width=.95pt,font=\scriptsize},
  offselected/.style={offblock,fill=offlinecolor!11},
  onblock/.style={draw=onlinecolor,fill=white,rounded corners=1pt,
    minimum width=1.00cm,minimum height=.52cm,inner sep=1pt,
    line width=.95pt,font=\scriptsize},
  onselected/.style={onblock,fill=onlinecolor!13,font=\scriptsize\bfseries},
  rankbadge/.style={circle,fill=offlinecolor,text=white,inner sep=0pt,
    minimum size=3.5mm,font=\tiny\bfseries}
]
\draw[black!22,line width=.55pt] (7.55,.22)--(7.55,5.48);

\node[anchor=west,font=\bfseries] at (.15,5.35) {(a) Offline top-prefix rule};
\node[anchor=west,font=\scriptsize,text=offlinecolor] at (.15,4.82)
  {compute all $V_{i,n}=I_i(n-1)$ and rank them globally};

\node[anchor=west,font=\scriptsize\bfseries] at (.18,4.15) {Initialization};
\node[offselected] at (2.75,4.15) {$(1,1)$};
\node[offselected] at (3.95,4.15) {$(2,1)$};
\node[offselected] at (5.15,4.15) {$(3,1)$};

\node[anchor=west,font=\scriptsize\bfseries] at (.18,3.65) {Post-initial items};
\foreach \x/\n in {1.55/2,3.05/3,4.55/4,6.05/5}{
  \node[font=\scriptsize\bfseries] at (\x,3.34) {$n=\n$};
}
\foreach \y/\i in {2.84/1,2.12/2,1.40/3}{
  \node[anchor=e,font=\scriptsize,text=black!68] at (.92,\y) {arm $\i$};
}
\node[offselected] at (1.55,2.84) {$(1,2)$}; \node[rankbadge] at (1.97,3.06) {2};
\node[offselected] at (3.05,2.84) {$(1,3)$}; \node[rankbadge] at (3.47,3.06) {4};
\node[offblock,text=black!58] at (4.55,2.84) {$(1,4)$};
\node[offblock,text=black!58] at (6.05,2.84) {$(1,5)$};
\node[offselected] at (1.55,2.12) {$(2,2)$}; \node[rankbadge] at (1.97,2.34) {1};
\node[offselected] at (3.05,2.12) {$(2,3)$}; \node[rankbadge] at (3.47,2.34) {5};
\node[offblock,text=black!58] at (4.55,2.12) {$(2,4)$};
\node[offblock,text=black!58] at (6.05,2.12) {$(2,5)$};
\node[offselected] at (1.55,1.40) {$(3,2)$}; \node[rankbadge] at (1.97,1.62) {3};
\node[offblock,text=black!58] at (3.05,1.40) {$(3,3)$};
\node[offblock,text=black!58] at (4.55,1.40) {$(3,4)$};
\node[offblock,text=black!58] at (6.05,1.40) {$(3,5)$};

\node[font=\scriptsize,text=offlinecolor,align=center] at (3.80,.54)
  {\textbf{select:} $(2,2),(1,2),(3,2),(1,3),(2,3)$};

\node[anchor=west,font=\bfseries] at (8.02,5.35) {(b) Online greedy rule};
\node[anchor=west,font=\scriptsize,text=onlinecolor] at (8.02,4.82)
  {choose an arm with the largest current value};

\node[anchor=west,font=\scriptsize\bfseries] at (8.02,4.15) {Initialization};
\node[onselected] at (10.60,4.15) {$(1,1)$};
\node[onselected] at (11.80,4.15) {$(2,1)$};
\node[onselected] at (13.00,4.15) {$(3,1)$};

\node[anchor=west,font=\scriptsize\bfseries] at (8.02,3.65)
  {Next available item from each arm};
\foreach \x/\t in {9.35/4,10.70/5,12.05/6,13.40/7,14.75/8}{
  \node[font=\scriptsize\bfseries] at (\x,3.29) {$t=\t$};
}
\foreach \y/\i in {2.84/1,2.12/2,1.40/3}{
  \node[anchor=e,font=\scriptsize,text=black!68] at (8.75,\y) {arm $\i$};
}

\node[onblock]    (q14) at (9.35,2.84) {$(1,2)$};
\node[onselected] (q24) at (9.35,2.12) {$(2,2)$};
\node[onblock]    (q34) at (9.35,1.40) {$(3,2)$};
\node[onselected] (q15) at (10.70,2.84) {$(1,2)$};
\node[onblock]    (q25) at (10.70,2.12) {$(2,3)$};
\node[onblock]    (q35) at (10.70,1.40) {$(3,2)$};
\node[onblock]    (q16) at (12.05,2.84) {$(1,3)$};
\node[onblock]    (q26) at (12.05,2.12) {$(2,3)$};
\node[onselected] (q36) at (12.05,1.40) {$(3,2)$};
\node[onselected] (q17) at (13.40,2.84) {$(1,3)$};
\node[onblock]    (q27) at (13.40,2.12) {$(2,3)$};
\node[onblock]    (q37) at (13.40,1.40) {$(3,3)$};
\node[onblock]    (q18) at (14.75,2.84) {$(1,4)$};
\node[onselected] (q28) at (14.75,2.12) {$(2,3)$};
\node[onblock]    (q38) at (14.75,1.40) {$(3,3)$};

\foreach \u/\v in {q14/q15,q34/q35,q25/q26,q35/q36,q16/q17,q26/q27,q27/q28,q37/q38}{
  \draw[->,black!28,line width=.5pt] (\u.east)--(\v.west);
}
\foreach \u/\v in {q24/q25,q15/q16,q36/q37,q17/q18}{
  \draw[->,onlinecolor,line width=.95pt] (\u.east)--(\v.west);
}

\node[font=\scriptsize,text=onlinecolor,align=center] at (11.575,.54)
  {\textbf{choose arm:} $2\rightarrow1\rightarrow3\rightarrow1\rightarrow2$};

\end{tikzpicture}
\endgroup
\caption{Offline and online selection ($K=3$, $T=8$).
An item \((i,n)\) represents the \(n\)-th pull of arm \(i\), has value \(V_{i,n}=I_i(n-1)\), and is post-initial when \(n\ge2\).
Consider a permutation $(\omega(1),\omega(2),\omega(3))=(2,3,1)$. Assume
$I_2(1)>I_1(1)>I_3(1)=I_1(2)>I_2(2)$ and all other indices are smaller.
For offline selection, after sorting $V_{i,n}=I_i(n-1)$ by decreasing value and resolving the tie using $\omega$, the first five items in the offline post-initial order are
$(2,2),(1,2),(3,2),(1,3),(2,3)$. For online selection, the first five choices after initialization are $2\to1\to3\to1\to2$, which corresponds to the same items $(2,2),(1,2),(3,2),(1,3),(2,3)$ as in the offline selection.}
\label{fig:online_offline_equivalence}
\end{figure}
Conditional on a reward array $Y$ and a permutation $\omega$, all stabilized indices are deterministic.
We introduce the item set $\mathcal I_T=\{(i,n):i\in[K],\,1\le n\le T\}$
and define two selection rules on it, where $(i,n)$ represents the $n$-th pull of arm $i$. 

For each item $(i,n)$, define its value as
$V_{i,n}=I_i(n-1)$, where $I_i(0)=\infty$.
Sort all items in $\mathcal I_T$ by decreasing value $V_{i,n}$, resolving ties by the smaller priority rank $\omega(i)$.

For a fixed number of rounds $T\ge K$, compare the following two selection rules:
\begin{itemize}[leftmargin=*]
    \item \textbf{Offline top-prefix rule.} Select top $T$ items in the sorted list. These consist of the $K$ initial items $\{(i,1)\}_{i\in[K]}$ and the top $T-K$ post-initial items.
    \item \textbf{Online greedy rule (\textup{\textsc{SLE-UCB}} selection rule).} Select $\{(i,1)\}_{i\in[K]}$ during the first $K$ rounds. At each later round $t>K$, choose an arm with the largest current value $I_i(N_i(t-1))$ and select the corresponding next item.
\end{itemize}
Figure~\ref{fig:online_offline_equivalence} illustrates the selection processes for both rules. With the following lemma, we show that the two rules select the same items.
\begin{lemma}[Online--offline equivalence]
\label{lem:online_offline_equivalence}
{For every reward array $Y$ and permutation $\omega$, the item set selected by the online greedy rule up to time $T\ge K$ is identical to the item set selected by the offline top-prefix rule.}
\end{lemma}

\begin{proof}
{Both rules select $\{(i,1)\}_{i\in[K]}$, so it remains to compare the remaining $T-K$ post-initial selections. Let $S_r^{\mathrm{off}}$ be the first $r$ post-initial items in the offline sorted list, and let $S_r^{\mathrm{on}}$ be the first $r$ post-initial items selected by the online greedy rule. We prove by induction that}
\begin{equation}
{
S_r^{\mathrm{on}}=S_r^{\mathrm{off}},\qquad 0\le r\le T-K.
}
\label{eq:online_offline_equivalence}
\end{equation}
The claim is immediate for $r=0$. Suppose it holds for some $r<T-K$, and let $(a,n)$ be the next item in the offline sorted list. By the induction hypothesis, the first $r$ post-initial items selected by the two rules form the same set. Hence, for each arm $i$, the offline rule has selected all items $(i,j)$ with $j\le N_i(K+r)$, and its highest-ranked unselected item from that arm is $(i,N_i(K+r)+1)$ because $V_{i,j}=I_i(j-1)$ is strictly decreasing in $j$. Consequently, $(a,n)$ is the item with the largest value among the items $(i,N_i(K+r)+1)$, $i\in[K]$, with ties resolved by $\omega$. At round $K+r+1$, the online greedy rule compares exactly these same items and uses the same tie-breaking rule, so it also selects $(a,n)$. This proves the induction step.
\end{proof}

{Directly under the online rule, changing the tie-breaking permutation or replacing a single reward may alter the next selected arm and thereby propagate through the entire subsequent decision trajectory, making the sensitivity of the final pull counts difficult to track. With the online--offline equivalence established, we use the offline top-prefix representation in both parts of the variance analysis. In Section~\ref{sec:term_star_bound}, it isolates the effect of random tie-breaking and shows that changing the priority permutation can alter the final count of any arm by at most one. In Section~\ref{sec:term_starstar_bound}, it converts a single-coordinate reward perturbation into a movement of the offline selection boundary, thereby avoiding a recursive analysis of subsequent online decisions.}

\subsection[Bound of Term (*) in the variance decomposition]{Bound of Term ($\star$) in \eqref{eq:var:decomp}}
\label{sec:term_star_bound}
\begin{proof}
Fix a clipped reward array $Y$, the target arm $i$, and two permutations $\omega$ and $\omega'$. By Lemma~\ref{lem:online_offline_equivalence}, under either permutation the algorithm selects the initial items $\{(j,1)\}_{j\in[K]}$ and then selects the top $T-K$ post-initial items according to the sorted item values, with ties resolved by the permutation. Let $c$ be the value of the last selected post-initial item.

Items with value strictly larger than $c$ are selected under both permutations, and items with value strictly smaller than $c$ are selected under neither permutation. Thus, the permutation can affect only items whose value is exactly $c$.
For the target arm $i$, according to~\eqref{eq:index_gap}, the sequence $\{I_i(n)\}_{n=1}^T$ is strictly decreasing. 
Hence arm $i$ has at most one post-initial item with value $c$, and changing the permutation can change the final count of arm $i$ by at most one:
$|N_i(Y,\omega,T)-N_i(Y,\omega',T)|\le1.$
With $\omega$ and $\omega'$ fixed, taking expectations over $Y$ gives
\begin{equation}
\left|
\mathbb E_Y[N_i(Y,\omega,T)\mid\omega]
-
\mathbb E_Y[N_i(Y,\omega',T)\mid\omega']
\right|\le
\mathbb E_Y\!\left[|N_i(Y,\omega,T)-N_i(Y,\omega',T)|\right]
\le1.
\label{eq:conditional_mean_interval}
\end{equation}
Therefore the conditional expectations lie in an interval of length at most one as $\omega$ varies. Since the target arm $i$ was arbitrary, for every $i\in[K]$,
\begin{equation}
\operatorname{Var}_{\omega}\!\left(\mathbb E_Y[N_i(Y,\omega,T)\mid \omega]\right)\le 1.
\label{eq:permutation_variance_bound}
\end{equation}
This provides an upper bound for Term ($\star$).
\end{proof}
\subsection[Bound of Term (**) in the variance decomposition]{Bound of Term ($\star\star$) in \eqref{eq:var:decomp}}
\label{sec:term_starstar_bound}
We fix the permutation $\omega$ and control the sensitivity of the final count of the target arm $i$ to a single-coordinate perturbation of the reward array. The proof has three steps: first, we bound the change in the stabilized indices caused by replacing one clipped reward; second, we translate this index perturbation into a pull-count perturbation; third, we sum the squared influences through the Efron--Stein inequality.
\begin{proof}
Consider clipped reward arrays $Y,Y'\in\mathbb R^{K\times T}$ differing only in the $m$-th clipped reward of arm $a$, where $Y'_{a,m}$ is an independent copy of $Y_{a,m}$. 
Let $\overline Y'_i(n)$, $U'_i(n)$, and $I'_i(n)$ denote the corresponding averages, lower-envelope UCB statistics, and stabilized indices computed under $Y'$.
For simplicity, since $\omega$ and $T$ are fixed in this subsection, we write $N_j(Y)=N_j(Y,\omega,T)$ for every arm $j\in[K]$.
\paragraph{Stabilized-index perturbation.} We first bound how a single-coordinate replacement changes the stabilized indices. For every arm $j\ne a$, the reward array is unchanged, so $I'_j(n)=I_j(n)$ for all $n\le T$. For arm $a$, the index is unchanged for $n<m$, because the replaced reward has not entered the prefix average. For $n\ge m$, we define the index attaining the minimum in $U_a(n)$ under $Y$ as
\begin{equation}
s(n)\in\operatorname*{arg\,min}_{1\le q\le n}\left\{\overline Y_a(q)+\frac{\beta}{\sqrt q}\right\}.
\label{eq:perturbation_minimizer}
\end{equation}
Using the definition of the lower-envelope UCB statistic, we have
\begin{equation}
U'_a(n)\le {\overline Y'_a(s(n))}+\frac{\beta}{\sqrt{s(n)}}\leq \overline Y_a(s(n))+\frac{|Y'_{a,m}-Y_{a,m}|}{s(n)}+\frac{\beta}{\sqrt{s(n)}}=U_a(n)+\frac{|Y'_{a,m}-Y_{a,m}|}{s(n)}.
\label{eq:one_sided_index_change}
\end{equation}
We next show that the minimizer $s(n)$ cannot be much smaller than $n$ on the concentration event $\mathcal G(Y)$.
Since $s(n)$ is a minimizer and $\mathcal G(Y)$ holds, we have
\begin{align}
\theta_a+\frac{\beta}{2\sqrt{s(n)}}
\le
\overline Y_a(s(n))+\frac{\beta}{\sqrt{s(n)}}
\le
\overline Y_a(n)+\frac{\beta}{\sqrt n}\le
\theta_a+\frac{3\beta}{2\sqrt n}.
\label{eq:recent_minimizer_chain_new}
\end{align}
The first and last inequalities use the event $\mathcal G(Y)$, and the chain implies $s(n)\ge n/9$.
\paragraph{Pull-count perturbation bound on $\mathcal G(Y)\cap\mathcal G(Y')$.}
Assume that $\mathcal G(Y)$ and $\mathcal G(Y')$ both hold. Applying $s(n)\ge n/9$ to $Y$, and then interchanging the roles of the two arrays, gives
\begin{equation}
I'_a(n)=I_a(n)\quad\text{if }n<m,
\qquad
|I'_a(n)-I_a(n)|\le \frac{9|Y'_{a,m}-Y_{a,m}|}{n}\quad\text{if }n\ge m.
\label{eq:index_perturbation_bound_new}
\end{equation}
We now convert the index perturbation in~\eqref{eq:index_perturbation_bound_new} into a perturbation bound for the final pull counts.
If $m>N_a(Y)$, then $I'_j(n)=I_j(n)$ for all $j\in[K]$ and $n\le N_j(Y)$, hence the selected item set is unchanged. Thus, for every $j\in[K]$,
\begin{equation}
    N_j(Y)=N_j(Y').
\label{eq:pull_count_unchanged}
\end{equation}
Similarly, if $m>N_a(Y')$, then $I_j(n)=I'_j(n)$ for all $j\in[K]$ and $n\le N_j(Y')$, which implies $N_j(Y)=N_j(Y')$ for every $j\in[K]$.

As for the case $m\le \min\{N_a(Y),N_a(Y')\}$, without loss of generality, we assume that $N_a(Y)<N_a(Y')$.
By Lemma~\ref{lem:online_offline_equivalence}, both runs are equivalent to top-prefix selections under the same permutation.
Thus, $(a,N_a(Y)+1)$ is not selected under $Y$, while $(a,N_a(Y'))$ is selected under $Y'$. Since the total number of selected items is fixed, there exists an unchanged arm $j$ such that the item $(j,\ell)$ is selected under $Y$ and not selected under $Y'$. Hence
\begin{equation}
I_a(N_a(Y))\le I_j(\ell-1)=I'_j(\ell-1)\le I'_a(N_a(Y')-1).
\label{eq:boundary_comparison_new}
\end{equation}
Subtracting $I_a(N_a(Y')-1)$ from both sides of \eqref{eq:boundary_comparison_new}, applying \eqref{eq:index_perturbation_bound_new}, and using the logarithmic gap of stabilized indices~\eqref{eq:index_gap}, we obtain
\begin{equation}
\begin{aligned}
\lambda \log\left(\frac{N_a(Y')-1}{N_a(Y)}\right)
&\le I_a(N_a(Y))-I_a(N_a(Y')-1)\\
&\le I'_a(N_a(Y')-1)-I_a(N_a(Y')-1)\\
&\le \frac{9|Y'_{a,m}-Y_{a,m}|}{N_a(Y')-1}.
\end{aligned}
\label{eq:pull_count_gap_1}
\end{equation}
{Using $\log(1+x)\geq x/(1+x)$ for $x\geq0$, we have
\begin{equation}
\lambda \cdot \frac{N_a(Y')-1-N_a(Y)}{N_a(Y')-1}
\le
\frac{9|Y'_{a,m}-Y_{a,m}|}{N_a(Y')-1}.
\label{eq:pull_count_gap_2}
\end{equation}
Exchanging the roles of $Y$ and $Y'$ gives the same bound when $N_a(Y)>N_a(Y')$. Therefore, when $\mathcal G(Y)$ and $\mathcal G(Y')$ both hold, we have
$|N_a(Y)-N_a(Y')|\le (1+9|Y'_{a,m}-Y_{a,m}|/\lambda)\mathbf 1_{\{m\le \min\{N_a(Y),N_a(Y')\}\}}.$}

For an unchanged arm $j\ne a$, remove arm $a$ and sort all unchanged-arm items in their common order. Conditional on the number of selected items from arm $a$, the selected items from the unchanged arms form a prefix of this remaining list. Hence the count change of every unchanged arm is bounded by the count change of arm $a$. Consequently, when $\mathcal G(Y)$ and $\mathcal G(Y')$ both hold,
{for every $j\in[K]$,}
\begin{equation}
{
|N_j(Y)-N_j(Y')|
{\le\left(1+\frac{9|Y'_{a,m}-Y_{a,m}|}{\lambda}\right)\mathbf 1_{\{m\le \min\{N_a(Y),N_a(Y')\}\}}.}
}
\label{eq:count_change_on_good_event_new}
\end{equation}
\paragraph{Expectation of the pull-count perturbation.}
Using~\eqref{eq:pull_count_unchanged},~\eqref{eq:count_change_on_good_event_new}, and the trivial bound
$|N_i(Y)-N_i(Y')|\le T$ when either $\mathcal G(Y)$ or $\mathcal G(Y')$ fails, we obtain
{
\begin{align}
&\mathbb E_{Y,Y'}\left[(N_i(Y)-N_i(Y'))^2\right]\\
&\qquad\leq\mathbb E_{Y,Y'}\big[(N_i(Y)-N_i(Y'))^2\mathbf 1_{\mathcal G(Y)\cap\mathcal G(Y')}\big]+\mathbb E_{Y,Y'}\big[(N_i(Y)-N_i(Y'))^2\mathbf 1_{\mathcal G(Y)^c\cup\mathcal G(Y')^c}\big]\\
&\qquad\leq \mathbb E_{Y,Y'}\Bigg[\Bigg(1+\frac{9|Y'_{a,m}-Y_{a,m}|}{\lambda}\Bigg)^2\Bigg]\mathbb P(m\leq N_a(Y))+4T^2(KT)^{-7}\label{eq:independence}\\
&\qquad\leq \bigg(1+\frac{9\sqrt{2}\sigma}{\lambda}\bigg)^2\mathbb P(m\leq N_a(Y))+4T^2(KT)^{-7}.
\label{eq:expectation_of_perturbation_new}
\end{align}
The first term in~\eqref{eq:independence} uses the independence of $\mathbf 1_{\{m\le \min\{N_a(Y),N_a(Y')\}\}}$ and $(Y_{a,m},Y'_{a,m})$: under $Y$, whether the $m$-th pull of arm $a$ is selected is determined before $Y_{a,m}$ is observed, and similarly for $Y'$.
The second term uses the union bound and Lemma~\ref{lem:clipped_concentration} to bound the probability of $\mathcal G(Y)^c\cup\mathcal G(Y')^c$. 
The last inequality uses Minkowski's inequality and bounds $\mathbb E_{Y,Y'}[(Y_{a,m}-Y'_{a,m})^2]$ by twice the variance of raw rewards.}
\paragraph{Conditional variance bound.} 
We now apply the Efron--Stein inequality in \citet{efron1981jackknife}, in the replacement form stated in
\citet[Theorem~3.1]{boucheron2013concentration} to the reward array and bound the conditional variance using the preceding pull-count estimate.

\begin{lemma}[Efron--Stein inequality]
\label{lem:efron_stein_new}
Let $X_1,\ldots,X_n$ be independent random variables, and let $X'_\ell$ be an independent copy of $X_\ell$ for all $\ell\in[n]$. Let $F=f(X_1,\ldots,X_n)$ be square-integrable. For $\ell\in[n]$, define $F^{(\ell)}=f(X_1,\ldots,X_{\ell-1},X'_\ell,X_{\ell+1},\ldots,X_n)$. Then we have
\begin{equation}
\operatorname{Var}(F)
\le
\frac12\sum_{\ell=1}^n
\mathbb E\left[(F-F^{(\ell)})^2\right].
\label{eq:efron_stein_new}
\end{equation}
\end{lemma}
{For a coordinate $(a,m)$ satisfying $a\in[K]$ and $m\in[T]$, let $Y^{(a,m)}$ be the array obtained by replacing $Y_{a,m}$ by an independent copy $Y'_{a,m}$ and leaving every other coordinate unchanged. Using Lemma~\ref{lem:efron_stein_new} and~\eqref{eq:expectation_of_perturbation_new}, with $\omega$ fixed, the conditional variance of the fixed target count can be bounded by}
\begin{align}
\operatorname{Var}_Y(N_i(Y,\omega,T)\mid\omega)
&\le
\frac12\sum_{a=1}^K\sum_{m=1}^T
\mathbb E_{Y,Y^{(a,m)}}\left[\left(N_i(Y,\omega,T)-N_i(Y^{(a,m)},\omega,T)\right)^2\right]\label{eq:conditional_variance_bound_1}\\
&{\leq \frac{(1+9\sqrt2\,\sigma/\lambda)^2}{2}\sum_{a=1}^K\sum_{m=1}^T\mathbb P_Y(m\le N_a(Y,\omega,T))+2T^2(KT)^{-6}}\label{eq:conditional_variance_bound_2}\\
&{=\frac{(1+9\sqrt2\,\sigma/\lambda)^2}{2}\sum_{a=1}^K\mathbb E_Y[N_a(Y,\omega,T)]+2T^2(KT)^{-6}}\label{eq:conditional_variance_bound_3}\\
&{=\frac{(1+9\sqrt2\,\sigma/\lambda)^2T}{2}+2T^2(KT)^{-6}\leq (1+9\sqrt2\,\sigma/\lambda)^2T}\label{eq:conditional_variance_bound_4}
\end{align}
{The first inequality~\eqref{eq:conditional_variance_bound_1} applies the Efron--Stein inequality over the replacement coordinates $(a,m)$. The second inequality~\eqref{eq:conditional_variance_bound_2} uses the single-coordinate perturbation bound in~\eqref{eq:expectation_of_perturbation_new}. The last line~\eqref{eq:conditional_variance_bound_4} uses $\sum_{a=1}^K\mathbb E_Y[N_a(Y,\omega,T)] = T$.}
Since $K,T\ge2$, we have $2T^2(KT)^{-6}\le T/2\le (1+9\sqrt{2}\sigma/\lambda)^2T/2$, so the second term in~\eqref{eq:conditional_variance_bound_2} is absorbed into the leading term.

Taking expectation over $\omega$, we can bound Term ($\star\star$) by
\begin{equation}
\mathbb E_{\omega}\!\left[\operatorname{Var}_Y(N_i(Y,\omega,T)\mid \omega)\right]
\le
{(1+9\sqrt2\,\sigma/\lambda)^2T}
.\label{eq:term_starstar_bound}
\end{equation}
This proves the desired bound for Term ($\star\star$).
\end{proof}

\subsection{Combining the bounds}
\label{sec:instability_combine}
Combining the variance decomposition~\eqref{eq:var:decomp}, the bound for Term ($\star$) in~\eqref{eq:permutation_variance_bound}, and the bound for Term ($\star\star$) in~\eqref{eq:term_starstar_bound}, we obtain that for every arm $i\in[K]$,
\begin{equation*}
  \operatorname{Var}(N_i(Y,\omega,T))\le 1+(1+9\sqrt2\,\sigma/\lambda)^2T \le 4\,(1+9\sqrt2\,\sigma/\lambda)^2T.
\end{equation*}
The last relaxation uses $T\ge1$. Substituting $\lambda=\alpha\beta\sqrt{K/T}$, we have
\begin{equation*}
  \mathcal{S}_{K,T}(\alpha)\le 2\sqrt T+\frac{18\sqrt 2\sigma T}{\alpha\beta\sqrt{K}}.
\end{equation*}
Because this bound is uniform over admissible instances \(\nu\) and arms \(i\), taking the maximum over arms and then the supremum over instances proves the instability bound in~\eqref{eq:main_performance_bounds}.

\section{\texorpdfstring{Proof of the \(K\)-armed lower bound}{Proof of the K-armed lower bound}}
\label{sec:lower_bound_comparison}

\begin{remark}[Fixed-\(K\) reduction]
\label{rem:fixed_k_reduction}
We establish the fixed-\(K\) extension in~\eqref{eq:chen_lu_fixed_k_implication} through the following padding reduction.
Pad the instance with \(K-2\) independent \(\mathcal N(-M,\sigma^2)\) dummy arms. A two-armed algorithm \(\mathcal B\) simulates a \(K\)-armed algorithm \(\mathcal A\) exactly on the first two arms. On dummy-arm requests, it alternates between the real arms from a uniformly random starting arm, discards the reward, and supplies an independent dummy reward. This gives \(\mathcal R_{2,T}^{\mathcal B}\le\mathcal R_{K,T}^{\mathcal A}\) and \(\mathcal S_{2,T}^{\mathcal B}\le\mathcal S_{K,T}^{\mathcal A}+1/2\) while preserving the required conditions. Since \(\mathcal R_{2,T}^{\mathcal B}\le\mathcal R_{K,T}^{\mathcal A}=o(T)\), their theorem gives \(\liminf_{T\to\infty}\mathcal R_{2,T}^{\mathcal B}\mathcal S_{2,T}^{\mathcal B}/T^{3/2}>C_0\) for some constant \(C_0>0\).
Together with the two inequalities above, this yields the fixed-\(K\) conclusion.
\end{remark}

\begin{proof}[Proof of Theorem~\ref{thm:lower_bound}]
Our proof largely follows the approach of \citet{chenlu2026}.
Fix $K\ge2$, \(T>K\). Let \(\mathcal A\) be an algorithm whose worst-case regret satisfies \(\mathcal R_{K,T}\le MT/8\).
Consider the all-zero Gaussian instance \(\nu(0)=(\mathcal N(0,\sigma^2))_{i=1}^K\), where the reward distributions of all arms are $\mathcal N(0,\sigma^2)$.
Choose an arm attaining the smallest expected pull count under this
instance and denote it by arm $j$.
{We consider the following family of instances, in which arm \(j\) is the unique optimal arm whenever \(\Delta>0\):
\begin{equation}
(\nu(\Delta))_i
=\begin{cases}
\mathcal N(0,\sigma^2), & i=j,\\
\mathcal N(-\Delta,\sigma^2), & i\ne j.
\end{cases}
\label{eq:equal_gap_gaussian_family}
\end{equation}
}
We aggregate all suboptimal pulls. Suppressing the dependence on $\mathcal A$, let $\mathbb P_\Delta$, $\mathbb E_\Delta$, and $\operatorname{Var}_\Delta$ denote probability, expectation, and variance under the interaction with $\nu(\Delta)$, and define
\begin{equation}
N_-(T):=\sum_{i\neq j} N_i(T)=T-N_j(T),
\qquad
G(\Delta):=\mathbb E_\Delta[N_-(T)].
\label{eq:aggregate_suboptimal_count}
\end{equation}
The equal-gap construction gives
\begin{equation}
\mathcal R_{K,T}(\nu(\Delta))=\Delta \cdot G(\Delta),
\qquad
G(0)=T-\mathbb E_0[N_j(T)]\ge\frac{K-1}{K}T,
\label{eq:equal_gap_regret_identity}
\end{equation}
The inequality for $G(0)$ follows because arm $j$ has the smallest expected pull count at $\Delta=0$.
Define a candidate gap by
\begin{equation}
\Delta_\star:=\frac{8\mathcal R_{K,T}}{T}.
\end{equation}
The regret requirement \(\mathcal R_{K,T}\le MT/8\) gives \(\Delta_\star\le M\), and hence \(\nu(\Delta_\star)\) is admissible.
We connect the all-zero instance to \(\nu(\Delta_\star)\) through a sequence of nearby instances. The increments are chosen so that the KL divergence between the interaction histories at consecutive gaps is at most \(1/2\). Set $\Delta_0=0$ and, recursively, define
\begin{equation}
\Delta_{\ell+1}
:=
\min\left\{
\Delta_\star,
\Delta_\ell+\frac{\sigma}{\sqrt{G(\Delta_\ell)}}
\right\}.
\label{eq:delta_chain_definition}
\end{equation}
At the zero instance, \(G(0)>0\), and hence $\mathbb P_0(N_-(T)>0)>0.$ Only the rewards from arms \(i\ne j\) differ between \(\nu(0)\) and \(\nu(\Delta)\). For every finite \(\Delta\), the chain rule for KL divergence and the Gaussian KL formula therefore give
$
D_{\rm KL}(\mathbb P_0\|\mathbb P_\Delta)
=G(0)\Delta^2/(2\sigma^2)
<\infty,$
so \(\mathbb P_0\ll\mathbb P_\Delta\). Consequently, \(\mathbb P_\Delta(N_-(T)>0)>0\), and thus \(G(\Delta)>0\). Therefore, the sequence \(\{\Delta_\ell\}_{\ell\ge0}\) is well defined.

Since $G(\Delta)\le T$, every untruncated increment in~\eqref{eq:delta_chain_definition} is at least $\sigma/\sqrt T$. Thus the sequence reaches $\Delta_\star$ after finitely many steps. Let $L$ be the first index at which the sequence reaches $\Delta_\star$, namely,
$L:=\min\{\ell\ge 1:\Delta_\ell=\Delta_\star\}.$
Since only the final increment can be truncated,
\begin{equation}
L
\le
1+\frac{\Delta_\star\sqrt T}{\sigma}
=
1+\frac{8\mathcal R_{K,T}}{\sigma\sqrt T}.
\label{eq:number_of_steps_bound}
\end{equation}
All points in the sequence belong to $[0,\Delta_\star]$, so all corresponding instances are admissible.
We next establish an order-$T$ gap between the endpoint values of $G$ along this sequence. By~\eqref{eq:equal_gap_regret_identity},
$
\Delta_\star \cdot G(\Delta_\star)
=
\mathcal R_{K,T}(\nu(\Delta_\star))
\le \mathcal R_{K,T}.
$
The definition of $\Delta_\star$ therefore gives
$
G(\Delta_\star)\le T/8.
$

Along this sequence, the triangle inequality implies
\begin{equation}
\sum_{\ell=0}^{L-1}
\left|G(\Delta_{\ell+1})-G(\Delta_\ell)\right|
\ge
\left|G(\Delta_L)-G(\Delta_0)\right|
=\left|G(\Delta_\star)-G(\Delta_0)\right|
\ge
\frac{T}{2}-\frac{T}{8}
=
\frac{3T}{8}.
\label{eq:telescoping_drop}
\end{equation}
Hence, for some $\ell\in\{0,\ldots,L-1\}$, 
$
\left|G(\Delta_{\ell+1})-G(\Delta_\ell)\right|
\ge
3T/(8L).
$

The following lemma is a variance-rescaled variant of Lemma~6.1 in \citet{chenlu2026}, applied to the aggregate count \(N_-(T)\). It compares two nearby Gaussian instances and converts the preceding gap into an instability lower bound.
\begin{lemma}
\label{lem:aggregate_pull_count_comparison}
Let $0\le\Delta<\Delta'\le M$ and $\Delta'\leq \Delta+\sigma/\sqrt{G(\Delta)}$, where $G(\Delta)$ is defined in~\eqref{eq:aggregate_suboptimal_count}. Under any algorithm,
\begin{equation}
\max\left\{
\mathcal S_{K,T}(\nu(\Delta)),
\mathcal S_{K,T}(\nu(\Delta'))
\right\}
\ge
\frac{e^{-1/4}}{4\sqrt{2}}
\left|G(\Delta')-G(\Delta)\right|.
\label{eq:aggregate_pull_count_comparison_bound}
\end{equation}
\end{lemma}
{See Section~\ref{sec:proof_aggregate_pull_count_comparison} for a detailed proof.}
Applying Lemma~\ref{lem:aggregate_pull_count_comparison} to the preceding gap and using \eqref{eq:number_of_steps_bound}, we obtain
\begin{align}
\mathcal S_{K,T}
&\ge
\max\left\{
\mathcal S_{K,T}(\nu(\Delta_\ell)),
\mathcal S_{K,T}(\nu(\Delta_{\ell+1}))
\right\}
\notag\\
&
\ge\frac{3e^{-1/4}T}{32\sqrt{2}L}\ge \frac{3e^{-1/4}}{256\sqrt{2}}
\cdot
\frac{T}{1+\mathcal{R}_{K,T}/(\sigma\sqrt T)}.
\label{eq:finite_time_lower_bound_k_arm}
\end{align}
The last inequality relaxes the bound on $L$ in~\eqref{eq:number_of_steps_bound} to $L\le 8(1+\mathcal{R}_{K,T}/(\sigma\sqrt T))$.
Multiplying~\eqref{eq:finite_time_lower_bound_k_arm} by $\mathcal{R}_{K,T}$ gives
\begin{equation}
\mathcal{R}_{K,T}\mathcal S_{K,T}
\ge
\frac{3e^{-1/4}\sigma}{256\sqrt{2}}
\cdot
\frac{\mathcal{R}_{K,T}}{\sigma\sqrt{T}+\mathcal{R}_{K,T}}
\cdot T^{3/2}.
\label{eq:finite_time_product_form}
\end{equation}

The standard Gaussian minimax construction in Exercise~15.2 of \citet{lattimore2020}, after scaling to variance $\sigma^2$ and restricting the means to $[-M,M]$, gives
\begin{equation}
\mathcal R_{K,T}
\ge
\min\left\{\frac{\sigma}{8},\frac{M}{4}\right\}\sqrt{KT}
=
\sigma c_{M,\sigma}\sqrt{KT},
\qquad
c_{M,\sigma}:=\min\left\{\frac18,\frac{M}{4\sigma}\right\}.
\label{eq:standard_minimax_lower_bound}
\end{equation}
{See Section~\ref{sec:proof_gaussian_minimax_bound} for a detailed proof.}
Since $K\ge2$ and $x\mapsto x/(1+x)$ is increasing on $[0,\infty)$, it follows that
$
\mathcal R_{K,T}/(\sigma\sqrt T+\mathcal R_{K,T})
\ge
c_{M,\sigma}\sqrt K/(1+c_{M,\sigma}\sqrt K)
\ge
\sqrt2\,c_{M,\sigma}/(1+\sqrt2\,c_{M,\sigma}).
$
Substituting it into~\eqref{eq:finite_time_product_form} yields the $K$-independent bound
\begin{equation}
\mathcal R_{K,T}\mathcal S_{K,T}
\ge
\frac{3e^{-1/4}\sigma c_{M,\sigma}}
{256\left(1+\sqrt2\,c_{M,\sigma}\right)}
T^{3/2}.
\end{equation}
\end{proof}

\begin{samepage}
\begin{remark}[Comparison of lower-bound conditions]
\label{rem:weaker_conditions}
Choosing an arm with the smallest expected pull count at $\Delta=0$ as the optimal arm in the instance construction yields the lower bound on $G(0)$ needed for~\eqref{eq:telescoping_drop}, which \citet{chenlu2026} obtain from symmetry. Absolute continuity of the Gaussian interaction laws implies $G(\Delta)>0$ for every finite $\Delta$, ensuring that the recursion in~\eqref{eq:delta_chain_definition} is well defined without nondegenerate exploration. Finally, under $\mathcal R_{K,T}\le MT/8$, $\Delta_\star\le M$, while the regret identity gives $G(\Delta_\star)\le T/8$. This admissible endpoint for fixed $T$, together with the lower bound on $G(0)$, yields the order-$T$ decrease in $G$ required by~\eqref{eq:telescoping_drop}, replacing the asymptotic active-learning argument of \citet{chenlu2026}.
\end{remark}
\end{samepage}

{
\subsection{\texorpdfstring{Proof of Lemma~\protect\ref{lem:aggregate_pull_count_comparison}}{Proof of the aggregate pull-count comparison}}
\label{sec:proof_aggregate_pull_count_comparison}}

\begin{proof}[Proof of Lemma~\ref{lem:aggregate_pull_count_comparison}]
The claim is immediate if $G(\Delta)=G(\Delta')$. Otherwise, define an event
\[
\mathcal A:=
\begin{cases}
N_-(T)<\dfrac{G(\Delta)+G(\Delta')}{2},
& G(\Delta)>G(\Delta'),\\[2mm]
N_-(T)\ge\dfrac{G(\Delta)+G(\Delta')}{2},
& G(\Delta)<G(\Delta').
\end{cases}
\]
On $\mathcal A$, $N_-(T)$ is at least
$|G(\Delta')-G(\Delta)|/2$ away from $G(\Delta)$, while on
$\mathcal A^c$ it is at least the same distance from $G(\Delta')$.
Since
$N_-(T)=T-N_j(T)$, it follows that
\begin{align*}
(\mathcal S_{K,T}(\nu(\Delta)))^2\ge \operatorname{Var}_\Delta(N_j(T))
&= \operatorname{Var}_\Delta(N_-(T))
\ge \mathbb P_\Delta(\mathcal A)\frac{|G(\Delta')-G(\Delta)|^2}{4}
.
\end{align*}
Similarly, we have $(\mathcal S_{K,T}(\nu(\Delta')))^2\ge\mathbb P_{\Delta'}(\mathcal A^c)\cdot |G(\Delta')-G(\Delta)|^2/4.$
Only the reward distributions of arms \(i\ne j\) differ between the two instances. The chain rule for KL divergence and the Gaussian KL formula therefore give
\[
D_{\rm KL}(\mathbb P_\Delta\|\mathbb P_{\Delta'})=\sum_{i\neq j}\mathbb E_{\Delta}[N_i(T)]D_{\rm KL}\big(\mathcal N(-\Delta,\sigma^2)\|\mathcal N(-\Delta',\sigma^2)\big)
=G(\Delta)\frac{(\Delta'-\Delta)^2}{2\sigma^2}
\le \frac12.
\]
The Bretagnolle--Huber inequality therefore implies
$
\mathbb P_\Delta(\mathcal A)+\mathbb P_{\Delta'}(\mathcal A^c)
\ge \exp\big(-D_{\rm KL}
(\mathbb P_\Delta\|\mathbb P_{\Delta'})\big)/2\ge e^{-1/2}/2.$ Consequently,
\begin{align*}
\max\left\{
\mathcal S_{K,T}(\nu(\Delta)),
\mathcal S_{K,T}(\nu(\Delta'))
\right\}
&\ge \frac{|G(\Delta')-G(\Delta)|}{4}
\left(\sqrt{\mathbb P_\Delta(\mathcal A)}
+\sqrt{\mathbb P_{\Delta'}(\mathcal A^c)}\right)\\
&\ge \frac{|G(\Delta')-G(\Delta)|}{4}
\sqrt{\mathbb P_\Delta(\mathcal A)
+\mathbb P_{\Delta'}(\mathcal A^c)}\ge \frac{e^{-1/4}}{4\sqrt2}
|G(\Delta')-G(\Delta)|.
\end{align*}
This concludes the proof of Lemma~\ref{lem:aggregate_pull_count_comparison}.
\end{proof}

{
\subsection{\texorpdfstring{Proof of~\protect\eqref{eq:standard_minimax_lower_bound}}{Proof of the Gaussian minimax bound}}
\label{sec:proof_gaussian_minimax_bound}}

\begin{proof}[Proof of the Gaussian minimax bound~\eqref{eq:standard_minimax_lower_bound}]\mbox{}\par
The proof applies the Gaussian minimax construction in Exercise~15.2 of \citet{lattimore2020}, with the variance scaled to \(\sigma^2\) and the means restricted to \([-M,M]\).
Set
\[
c_{M,\sigma}:=\min\left\{\frac18,\frac{M}{4\sigma}\right\},
\qquad
\delta:=4c_{M,\sigma}\sigma\sqrt{\frac KT}.
\]
For each $i\in[K]$, consider the Gaussian instance $\nu^{(i)}$ in which
arm $i$ has mean $\delta$ and all other arms have mean zero, with
variance $\sigma^2$ throughout; let $\nu^{(0)}$ be the all-zero instance.
These instances are admissible because $\delta\le M$.
Writing $\mathbb P_{\nu^{(i)}}$, $\mathbb E_{\nu^{(i)}}$, $R_{\nu^{(i)}}$ for the corresponding probability, expectation and regret under the interaction with instance $\nu^{(i)}$, Pinsker's inequality and the divergence decomposition give
\begin{align}
\mathbb E_{\nu^{(i)}}[N_i(T)]
&\le \mathbb E_{\nu^{(0)}}[N_i(T)]+
T\cdot\operatorname{TV}(\mathbb P_{\nu^{(i)}},\mathbb P_{\nu^{(0)}})
\le \mathbb E_{\nu^{(0)}}[N_i(T)]
+T\sqrt{\frac12D_{\rm KL}(\mathbb P_{\nu^{(0)}}\|\mathbb P_{\nu^{(i)}})}\\
&=\mathbb E_{\nu^{(0)}}[N_i(T)]
+2c_{M,\sigma}\sqrt{KT\,\mathbb E_{\nu^{(0)}}[N_i(T)]},
\end{align}
where $\operatorname{TV}(\mathbb P_{\nu^{(i)}},\mathbb P_{\nu^{(0)}})
=\sup_A|\mathbb P_{\nu^{(i)}}(A)-\mathbb P_{\nu^{(0)}}(A)|.$
Summing over $i$ and applying the Cauchy--Schwarz inequality yields
\[
\sum_{i=1}^K\mathbb E_{\nu^{(i)}}[N_i(T)]\le T+2c_{M,\sigma}KT.
\]
Hence, with $R_{\nu^{(i)}}=\delta(T-\mathbb E_{\nu^{(i)}}[N_i(T)])$,
\begin{align*}
\mathcal R_{K,T}
&\ge \frac1K\sum_{i=1}^K R_{\nu^{(i)}}
\ge 4c_{M,\sigma}\sigma\left(1-\frac1K-2c_{M,\sigma}\right)\sqrt{KT}\ge c_{M,\sigma}\sigma\sqrt{KT}
=\min\left\{\frac{\sigma}{8},\frac{M}{4}\right\}\sqrt{KT},
\end{align*}
where the last inequality uses $K\ge2$ and $c_{M,\sigma}\le1/8$.
\end{proof}

\section{Conclusion}
\label{sec:conclusion}

We characterize the finite-time regret--instability trade-off for \(K\)-armed bandits. Our lower bound has a constant independent of \(K\), while the tunable \textup{\textsc{SLE-UCB}} algorithm achieves a regret--instability product of \(O(T^{3/2}\log K)\). These bounds match in their \(T^{3/2}\) dependence and up to a logarithmic factor in \(K\), thereby characterizing the frontier to logarithmic accuracy. In adaptive experimentation and dynamic resource allocation, the tuning parameter lets a decision maker choose how much regret to accept for more reproducible allocations, while the lower bound shows that this trade-off is unavoidable within the model. It would be interesting to investigate whether these ideas extend to contextual bandits and reinforcement learning.

\bibliographystyle{ims}
\bibliography{graphbib}

@string{COLT  = {Conference on Learning Theory}}

@article{chenlu2026,
  author = {Chen, Yilun and Lu, Jiaqi},
  title = {Bandit Allocational Instability},
  journal = {arXiv preprint arXiv:2602.07472},
  year = {2026}
}

@book{lattimore2020,
  title={Bandit algorithms},
  author={Lattimore, Tor and Szepesv{\'a}ri, Csaba},
  year={2020},
  publisher={Cambridge University Press}
}

@incollection{boucheron2013concentration,
    author = {Boucheron, Stéphane and Lugosi, Gábor and Massart, Pascal},
    isbn = {9780199535255},
    title = {Bounding the Variance},
    booktitle = {Concentration Inequalities: A Nonasymptotic Theory of Independence},
    publisher = {Oxford University Press},
    pages = {52--82},
    year = {2013},
    month = {02},
}

@article{efron1981jackknife,
author = {B. Efron and C. Stein},
title = {{The Jackknife Estimate of Variance}},
volume = {9},
journal = {The Annals of Statistics},
number = {3},
publisher = {Institute of Mathematical Statistics},
pages = {586 -- 596},
year = {1981},
}

@article{praharaj2025instability,
  title={On instability of minimax optimal optimism-based bandit algorithms},
  author={Praharaj, Samya and Khamaru, Koulik},
  journal={arXiv preprint arXiv:2511.18750},
  year={2025}
}

@article{kalvit2021closer,
  title={A closer look at the worst-case behavior of multi-armed bandit algorithms},
  author={Kalvit, Anand and Zeevi, Assaf},
  journal={Advances in Neural Information Processing Systems},
  volume={34},
  pages={8807--8819},
  year={2021}
}

@article{khamaru2024inference,
  title={Inference with the upper confidence bound algorithm},
  author={Khamaru, Koulik and Zhang, Cun-Hui},
  journal={arXiv preprint arXiv:2408.04595},
  year={2024}
}

@article{han2024ucb,
  title={UCB algorithms for multi-armed bandits: Precise regret and adaptive inference},
  author={Han, Qiyang and Khamaru, Koulik and Zhang, Cun-Hui},
  journal={arXiv preprint arXiv:2412.06126},
  year={2024}
}

@article{fan2022typical,
  title={The typical behavior of bandit algorithms},
  author={Fan, Lin and Glynn, Peter W},
  journal={arXiv preprint arXiv:2210.05660},
  year={2022}
}

@article{halder2025stable,
  title={Stable thompson sampling: Valid inference via variance inflation},
  author={Halder, Budhaditya and Pan, Shubhayan and Khamaru, Koulik},
  journal={arXiv preprint arXiv:2505.23260},
  year={2025}
}

@article{fan2026precise,
  title={Precise asymptotics and refined regret of variance-aware ucb},
  author={Fan, Yingying and Han, Yuxuan and Lv, Jinchi and Xu, Xiaocong and Zhou, Zhengyuan},
  journal={Advances in Neural Information Processing Systems},
  volume={38},
  pages={160795--160837},
  year={2025}
}

@article{lai1982least,
  title={Least squares estimates in stochastic regression models with applications to identification and control of dynamic systems},
  author={Lai, Tze Leung and Wei, Ching Zong},
  journal={The Annals of Statistics},
  pages={154--166},
  year={1982},
  publisher={JSTOR}
}

@misc{halder2026stabilizingbanditsusingregularization,
      title={Stabilizing Bandits using Regularization: Precise Regret and A Quantitative Central Limit Theorem}, 
      author={Budhaditya Halder and Ishan Sengupta and Koustav Chowdhury and Samya Praharaj and Koulik Khamaru},
      year={2026},
      eprint={2603.10184},
      archivePrefix={arXiv},
      primaryClass={stat.ML},
}

@article{yan2026optimism,
  title={Optimism Stabilizes Thompson Sampling for Adaptive Inference},
  author={Yan, Shunxing and Zhong, Han},
  journal={arXiv preprint arXiv:2602.06014},
  year={2026}
}

@book{gittins2011,
  title={Multi-armed bandit allocation indices},
  author={Gittins, John and Glazebrook, Kevin and Weber, Richard},
  year={2011},
  publisher={John Wiley \& Sons}
}

@article{villar2015,
  title={Multi-armed bandit models for the optimal design of clinical trials: benefits and challenges},
  author={Villar, Sof{\'\i}a S and Bowden, Jack and Wason, James},
  journal={Statistical science: a review journal of the Institute of Mathematical Statistics},
  volume={30},
  number={2},
  pages={199},
  year={2015}
}

@inproceedings{li2010,
  title={A contextual-bandit approach to personalized news article recommendation},
  author={Li, Lihong and Chu, Wei and Langford, John and Schapire, Robert E},
  booktitle={Proceedings of the 19th international conference on World wide web},
  pages={661--670},
  year={2010}
}

@inproceedings{audibert2009minimax,
  title={Minimax policies for adversarial and stochastic bandits},
  author={Audibert, Jean-Yves and Bubeck, S{\'e}bastien},
  booktitle={Colt},
  pages={217--226},
  year={2009}
}

@article{fan2025fragility,
  title={The fragility of optimized bandit algorithms},
  author={Fan, Lin and Glynn, Peter W},
  journal={Operations Research},
  volume={73},
  number={6},
  pages={3173--3198},
  year={2025},
  publisher={INFORMS}
}

@article{simchilevi2022simple,
  title={A simple and optimal policy design for online learning with safety against heavy-tailed risk},
  author={Simchi-Levi, David and Zheng, Zeyu and Zhu, Feng},
  journal={Advances in Neural Information Processing Systems},
  volume={35},
  pages={33795--33805},
  year={2022}
}

@article{simchilevi2023tradeoff,
  title={Stochastic multi-armed bandits: Optimal trade-off among optimality, consistency, and tail risk},
  author={Simchi-Levi, David and Zheng, Zeyu and Zhu, Feng},
  journal={Advances in Neural Information Processing Systems},
  volume={36},
  pages={35619--35630},
  year={2023}
}

@inproceedings{simchilevi2023experimental,
  title={Multi-armed bandit experimental design: Online decision-making and adaptive inference},
  author={Simchi-Levi, David and Wang, Chonghuan},
  booktitle={International Conference on Artificial Intelligence and Statistics},
  pages={3086--3097},
  year={2023},
  organization={PMLR}
}

@article{chenlu2025adaptivity,
  author = {Chen, Yilun and Lu, Jiaqi},
  title = {A Characterization of Sample Adaptivity in UCB Data},
  journal = {arXiv preprint arXiv:2503.04855},
  year = {2025}
}

@article{han2026thompson,
  title={Thompson sampling: Precise arm-pull dynamics and adaptive inference},
  author={Han, Qiyang},
  journal={arXiv preprint arXiv:2601.21131},
  year={2026}
}

@article{zhu2026varianceinflation,
  title={Adaptive variance inflation in thompson sampling: Efficiency, safety, robustness, and beyond},
  author={Zhu, Feng and Simchi-Levi, David},
  journal={Advances in Neural Information Processing Systems},
  volume={38},
  pages={50466--50484},
  year={2025}
}

@article{auer2002finite,
  title={Finite-time analysis of the multiarmed bandit problem},
  author={Auer, Peter and Cesa-Bianchi, Nicolo and Fischer, Paul},
  journal={Machine learning},
  volume={47},
  number={2},
  pages={235--256},
  year={2002},
  publisher={Springer}
}

\end{document}